\documentclass[12pt]{article}
\usepackage[T1]{fontenc}
\usepackage[latin9]{inputenc}
\usepackage[english]{babel}
\usepackage{url}
\usepackage{bm}
\usepackage{amsmath}
\usepackage{amssymb}
\usepackage{amsthm}
\usepackage{mathtools}
\usepackage{amsmath}
\usepackage{amssymb}
\usepackage{stmaryrd}
\usepackage{babel}
\usepackage{comment}

\usepackage{graphicx}
\usepackage{natbib}
\usepackage[font=singlespacing]{caption}
\newtheorem{theorem}{Theorem}
\usepackage{wasysym}
\usepackage{adjustbox}
\usepackage[pagewise]{lineno}\linenumbers
\modulolinenumbers[5]
\nolinenumbers
\usepackage{soul}

\usepackage{adjustbox}
\usepackage{enumitem}
\usepackage[section]{placeins}
\usepackage{breakcites}
\usepackage{booktabs}
\usepackage{amsthm}
\newtheorem{lemma}{Lemma}
\newtheorem{corollary}{Corollary}
\newtheorem{prop}{Proposition}
\newtheorem{defn*}{Definition}
\usepackage{thmtools}
\usepackage{colortbl}

\usepackage{url}
\usepackage{color}
\usepackage[ruled,linesnumbered,resetcount]{algorithm2e}
\usepackage{algcompatible}
\newcounter{algoline}
\newcommand{\blind}{1}

\usepackage[unicode=true,
 bookmarks=false,
 breaklinks=true,pdfborder={0 0 1},backref=false,colorlinks=false]
 {hyperref}

\providecommand{\tabularnewline}{\\}

\usepackage{xcolor}

\newcommand{\hrluonote}[1]{{\color{blue}[ #1 -- HL ]}}

\newcommand{\jsnote}[1]{{\color{magenta}[ #1 -- JS ]}}

\newlength\myindent
\newcommand{\centered}[1]{\begin{tabular}{l} #1 \end{tabular}}

\begin{document}

\def\spacingset#1{\renewcommand{\baselinestretch}{#1}\small\normalsize} \spacingset{1}

\if1\blind
{
  \title{\bf Nonparametric Multi-shape Modeling with Uncertainty Quantification}%
  \author{Hengrui Luo\\
Lawrence Berkeley National Laboratory\\
and \\
    Justin D. Strait\\
Los Alamos National Laboratory}
\date{}
\maketitle
  } \fi

\if0\blind
{
  \bigskip
  \bigskip
  \bigskip
  \begin{center}
    {\LARGE\bf Nonparametric Multi-shape Modeling with Uncertainty Quantification}
\end{center}
  \medskip
} \fi

\begin{abstract}
Modeling and uncertainty quantification of closed curves is an important problem in the field of shape analysis, and can have significant ramifications for subsequent statistical tasks. Characterization of dependence for closed manifolds is a challenging task. Furthermore, many shape analysis tasks involve collections of closed curves, which often exhibit structural similarities at multiple levels. Modeling multiple closed curves in a way that efficiently incorporates such between-curve dependence remains a challenging problem.
In this work, we propose and investigate a multiple-output, 
multi-dimensional Gaussian process modeling framework. We illustrate the proposed methodological advances, and demonstrate the utility of meaningful uncertainty quantification, on several curve and shape-related tasks. This model-based approach not only addresses the problem of inference on closed curves (and their shapes) through kernel constructions, but also opens doors to nonparametric modeling of multi-level dependence for functional objects in general.
\end{abstract}


\spacingset{1.9} 

\section{Problem Formulation}
\label{sec:Introduction}
\subsection{Background}
Shape data are ubiquitous in modern applications, often extracted as object boundaries from images. Typically, shapes are either represented as finite point sets \citep{kendall_shape,dryden2016statistical} or planar curves \citep{srivastava2016functional}.
Geometrically, planar curves are either open (i.e., there exist uniquely-defined starting and ending points) or closed (i.e., any arbitrarily defined starting point on the curve is also its ending point). %
The focus of this paper is on using noisy realizations of points to estimate underlying closed planar curves 
via Gaussian process (GP) models, and its ensuing impact on various shape analysis tasks.%

Closed curves can be viewed as a type of multi-dimensional functional data, with closed input domains \citep{olsen2018simultaneous}. %
Modern shape analysis models characterize variation across curves, yet tend to ignore the uncertainty in curve fitting to observed realizations, as these points are generally assumed to be noise-free, \emph{dense} samples along the curve \citep{kurtek2012statistical,huang2015riemann}. Application of these models to noisy dense samples (e.g., points extracted from object boundaries in images) or sparse samples (e.g., expert landmark annotations) may not properly account for this additional uncertainty (Supplementary Material \ref{sec:ill_param}). Unfortunately, current approaches to uncertainty quantification (UQ) of such complex geometric objects focus primarily on either open manifolds without boundaries, or univariate functional data (e.g., fMRI time series \citep{luo2021topologicaltransfer}). UQ has been less extensively studied for curves parameterized on closed manifolds, e.g., $\mathcal{S}^1$.
%

%
%

%
%
%
%

\subsection{Challenges for Modeling Approaches}
\label{sec:Challenges}

Both univariate functional data and closed planar curves can be viewed as parametric functions, where points in $d$-dimensional Euclidean space ($d\geq 1$) are indexed by a parameter value. 
Although there is an intrinsic relationship between closed curves %
and univariate functions \citep{olsen2018simultaneous,li2017bayesian}, there are 
two major considerations required to extend functional data models to closed curves:

\textbf{\emph{Closedness.}} The input domain for a closed curve is the unit circle $\mathcal{S}^1$, since any arbitrarily chosen starting point on it is also its ending point, unlike  an open 
curve, where the domain is $\mathbb{R}$ (or a subset). %
Parameterizations on $\mathcal{S}^1$ present modeling challenges, %
as further assumptions are required to respect the closed nature of curves (which we deem \emph{consistency}).
Closed manifold models which ignore boundaries can introduce non-trivial numerical and consistency issues. We address this for GP models by defining a periodic covariance kernel \citep{duvenaud2014automatic}, carefully restricting the
relationship between its length-scale and periodicity hyperparameters to ensure consistency in estimation.

\textbf{\emph{Dependence.}} %
GP models for functional data analysis \citep{shi2011gaussian,wang2016functional,greven2017general,olsen2018simultaneous} usually consider univariate functions with monotonic parameterizations. For closed planar curves, %
models should be specified simultaneously within a single coordinate (one-dimensional output) on both $x$- and $y$-coordinates (two-dimensional output) as functions of one shared parameter (one-dimensional input), with additional dependence accounted for between curves in the presence of
multiple curves. %
We account for this by designing a multi-level, multiple-output kernel, adapting the sampling scheme to closed curves. %
%

%
Statistical approaches to modeling closed curves are either model-based or model-free.  A \emph{model-based approach} defines a generative, probabilistic model. 
Some examples of model-based approaches for functional data and curves include \citet{stocker2021functional,greven2017general,lu2017bayesian,strait2019lmk}.
This is in contrast to \emph{model-free approaches}, which do not allow probabilistic inference %
due to the lack of a data-generating mechanism. %
Examples include noise-free spline models without uncertainty %
\citep{steyer2021elastic, luo2019combining}, and
shape registration and estimation of population shape summaries (e.g., Karcher mean/variance) as in \cite{kendall_shape,srivESA}. 
These do not allow explicit modeling of the dependence between curve locations or curves.

In this paper, we propose the use of a multiple-output GP model (a model-based approach), with a novel multi-level covariance kernel, for estimation, prediction, and uncertainty quantification of \emph{planar closed curves}. 
The \emph{multi-level dependence structure} captures the dependence between points on closed curves (within-curve dependence) and structural similarity across collections of closed curves (between-curve dependence), %
GP models naturally incorporate UQ, which can be used for subsequent shape analysis tasks. 

\section{Closed Curves, Shapes, and Parameterizations}
\label{sec:ClosedCurves}

We first introduce the data format of \emph{closed curves}. %
A parameterized curve $\mathbf{f}: \mathcal{D} \rightarrow \mathbb{R}^d$ is a mapping from curve parameter domain $\mathcal{D}$ into $d$-dimensional Euclidean space. We focus on planar curves, where $d=2$, such that curve $\mathbf{f}$ has coordinate functions $f_1, f_2$ with univariate scalar outputs. %
Assume that the curve is of total length $\ell$. Open curves satisfy $\mathbf{f}(0) \neq \mathbf{f}(\ell)$, whereas closed curves satisfy $\mathbf{f}(0)=\mathbf{f}(\ell)$. %
This consistency condition means that the domain of a closed curve is diffeomorphic to $\mathcal{S}^1$.

A central question in modeling is how to parameterize closed curves. It is most common to restrict to parameterizations which are \emph{monotonic}, since %
these ensure that there is no stopping or reversal of motion as the curve is traversed. As a result, each point is uniquely identified by a parameter value on $[0,\ell)$, %
avoiding the case where the same point corresponds to two different parameter values \citep{wang2016functional,srivESA}. 
The most natural, monotonic parameterization that we consider for our proposed GP model is the \emph{arc-length parameterization}. Given a curve $\mathbf{f}$, its 
\emph{arc-length} is the function $s(t)=\int_{0}^{t}\vert \mathbf{f}'(u)\vert\ du$,
which evaluates the length of $\mathbf{f}$ up to  %
parameter value $t$. Then, %
$s=s(t)$  can be used as a new parameter, and $\mathbf{f}(s)$ is arc-length parameterized, with the form $s\mapsto(f_{1}(s),f_{2}(s))$, where $s$ varies from $0$ to the total length $\ell$. %

While closed curves do not have a natural starting or ending point, they are often provided as an ordered set of points. 
We can identify a starting point, and associate it with both $s=0$ and $s=\ell$, due to the closedness of the curve. 
Unlike equally-spaced approximation dense samples, properly computing arc-length parameter values for sparsely-sampled curves is %
of greater importance and less explored as shown by  Figure \ref{fig:twotwo_table}, along with algorithmic error bounds in Theorem \ref{thm:con_xy_arc} (Supplementary Material \ref{sec:arc-length Parameterization Algorithms}).

Curves can be parameterized in other ways. For instance, in elastic shape analysis
\citep{srivESA,kurtek2012statistical}, the problem of registering curves involves optimizing (with respect to the \emph{elastic metric}) over a group of re-parameterizations. \
Unfortunately, 
this optimization can result in unsatisfactory curve alignments, especially if sampled points are sparse or noisy (Supplementary Material \ref{sec:numerical_issues}), resulting in problematic interpretations for downstream shape analysis tasks.
Thus, we suggest using arc-length parameterizations for the proposed closed curve model fitting, as it appears to be more widely suitable regardless of sampling scheme.  We present a detailed comparison of the arc-length and ``elastic-induced'' parameterizations in Supplementary Material \ref{sec:ill_param}.

%
%
%

%
%
%
%
%

%
\section{Multiple-output Model for Closed Curves}
\label{sec:SingleCurve}

In this section, we propose a multiple-output GP model to fit a collection of closed curves, which addresses the challenges presented in Section \ref{sec:Introduction}. A basic discussion of GPs can be found in the Supplementary Material \ref{sec:SOGP} and \citet{luo2019sparse}.

\subsection{Data Structure and Notation}
\label{subsec:DataNot}

Suppose that we observe sample points from $J$ closed curves, all of which we assume are arc-length parameterized. In other words, for curve $j=1,\ldots,J$, we observe $n_j$ input-output combinations $(s_{i}^{(j)},\mathbf{y}_{i\cdot}^{(j)})$, where $s_{i}^{(j)}\in\mathbb{R}$ is the arc-length parameter value associated with point $\mathbf{y}_{i\cdot}^{( j)}=(y_{i1}^{(j)},y_{i2}^{(j)})^T\in\mathbb{R}^{2}$ (as continuous variables). This point is assumed to have been sampled from the $j$-th (unknown) underlying  curve $\mathbf{f}^{(j)}=\left(f_{1}^{(j)},f_{2}^{(j)}\right)$.
The number of observed sample points $n_j$ can vary by curve, such that the vector of arc-length parameters $\mathbf{s}^{(j)}=(s_{1}^{(j)},\ldots,s_{n_j}^{(j)})^T$ is computed with respect to curve $j$. 
Let $\mathbf{Y}^{(j)}$ contain the full two-dimensional output values for curve $j$:
\begin{align}
\mathbf{Y}^{(j)}=\begin{pmatrix}y_{11}^{(j)} & y_{12}^{(j)}\\
\vdots & \vdots\\
y_{n_j,1}^{(j)} & y_{n_j,2}^{(j)}
\end{pmatrix}=\begin{pmatrix}\mathbf{y}_{1\cdot}^{(j)T}\\
\vdots\\
\mathbf{y}_{n_j\cdot}^{(j)T}
\end{pmatrix}=\begin{pmatrix}\mathbf{y}_{\cdot1}^{(j)} & \mathbf{y}_{\cdot2}^{(j)}\end{pmatrix}\in\mathbb{R}^{n_j\times2} \,.
\label{eq:coordinate convetion}
\end{align}
To clarify notation, for curve $j$, $\mathbf{y}_{i\cdot}^{(j)}\in\mathbb{R}^{2}$ represents the Euclidean coordinates of the output at input parameter $s_{i}^{(j)}$, whereas $\mathbf{y}_{\cdot d}^{(j)}\in\mathbb{R}^{n_j}$ represents all $n_j$ output values across dimensions $d\in\{1,2\}$.
We reiterate that while the goal is to fit a closed curve, which by definition has no natural starting point, the point data used in fitting inherently is provided as an array with its first row as a ``starting point'' (i.e., $\mathbf{y}_{1\cdot}^{(j)}$). The fitted curve should be invariant to permutation of row ordering of the original point set, which we demonstrate in Figures \ref{fig:perm_sp} and \ref{fig:perm_sp2} of Supplementary Material \ref{subsec:order}.

Throughout the rest of this section, we will either refer to observed sample points from curve $j$ in matrix form, i.e., \eqref{eq:coordinate convetion}, or in the vectorized form below:
\begin{equation}
\text{vec}\left(\mathbf{Y}^{(j)}\right) = \left( \mathbf{y}_{1\cdot}^{(j)T}, \hdots, \mathbf{y}_{n_{j}\cdot}^{(j)T} \right)^T = \left( y_{11}^{(j)}, y_{12}^{(j)}, \hdots, y_{n_{j},1}^{(j)}, y_{n_{j},2}^{(j)} \right)^T \in \mathbb{R}^{2n_j} \,.
\end{equation}
This is structurally more useful and convenient when considering joint models across curves.

\subsection{Single-output Gaussian Processes for a Single Curve}
\label{subsec:Setup}

First, consider the following general model with additive noise for a single curve $f^{(j)}$:
\begin{align}
\left(\begin{array}{c}
y_{i1}^{(j)}\\
y_{i2}^{(j)}
\end{array}\right)=\left(\begin{array}{c}
f_{1}^{(j)}(s_{i}^{(j)})\\
f_{2}^{(j)}(s_{i}^{(j)})
\end{array}\right)+\left(\begin{array}{c}
\epsilon_{i1}^{(j)}\\
\epsilon_{i2}^{(j)}
\end{array}\right)\ \Longleftrightarrow \ \mathbf{y}_{i\cdot}^{(j)}=\mathbf{f}^{(j)}(s_{i}^{(j)})+\bm{\epsilon}_{i}^{(j)}\,,\label{eq:within_coords_dependence}
\end{align}
where $\boldsymbol{\epsilon}_{i\cdot}^{(j)}=(\epsilon_{i1}^{(j)},\epsilon_{i2}^{(j)})^T$ is the error vector and $\mathbf{f}^{(j)}=(f_1^{(j)},f_2^{(j)})$ is unknown. Note that
$\mathbf{f}^{(j)}$ maps one-dimensional input (i.e., arc-length) to a two-dimensional output (i.e., $x$- and $y$-coordinates). In this paper, we view curve $j$'s coordinate functions $\left(f_{1}^{(j)}(s),f_{2}^{(j)}(s)\right)$ directly as $(x,y)$ coordinates in $\mathbb{R}^{2}$. This allows uncertainty estimates to be directly interpretable as a function of location %
with respect to the coordinate system along the $x$- and $y$-axes.%

A naive way to extend single-output GPs to multiple dimensions is to fit separate, independent GPs to each output dimension (i.e., coordinate function), so that we obtain the finite-dimensional model in \eqref{eq:within_coords_dependence}.  We deem this the \emph{baseline model}: $f_{1}^{(j)},f_{2}^{(j)}$ are %
modeled by GPs with kernel functions $k_{1}^{(j)},k_{2}^{(j)}$ and $\epsilon_1^{(j)},\epsilon_2^{(j)}$ are independent Gaussian noise terms.
The dependence within the $d{\text{-th}}$ coordinate function, for $d=1,2$, is 
$\text{Cov}\left(f_{d}^{(j)}(s_i^{(j)}),f_{d}(s_p^{(j)})\right)=k_{d}^{(j)}\left(s_{i}^{(j)},s_{p}^{(j)}\right)$
for $n_j$ different sample points from the curve indexed by $i,p=1,\ldots,n_j$, which allows $k_1^{(j)}$ and $k_2^{(j)}$ to be independently specified. %
For the vector $\mathbf{f}^{(j)}(s_i^{(j)})$, the joint kernel %
produces a  $2\times 2$ covariance matrix,  
$K_{\mathbf{f}^{(j)}}(s_i^{(j)},s_p^{(j)})=\text{diag}\left( k_1^{(j)}(s_i^{(j)},s_p^{(j)}), k_2^{(j)}(s_i^{(j)},s_p^{(j)}) \right)$,
and the covariance matrix for $\bm{\epsilon}_i^{(j)}$ can be written as $\Sigma^{(j)}=\text{diag}(\sigma_{\epsilon1}^{2(j)},\sigma_{\epsilon2}^{2(j)})$. Zeros on the off-diagonals of these covariance matrices imply that there is no dependence between coordinate functions. 

Using the maximum likelihood GP model fitting procedure (discussed in Supplementary Material \ref{sec:general_numericals} and \ref{sec:restricted parameter range}), Figure \ref{fig:sogp_fit} shows a baseline model fit for a bone from the MPEG-7 dataset \citep{bai2009learning}. The left and middle panels show the predictive mean and 95\% confidence intervals %
for the estimated  $x$- and $y$-coordinate functions. The predictive mean and variances can be combined to obtain the GP model estimates for the closed curve shown in the right panel, with uncertainty (as represented by ellipsoid volumes) which shrinks near observed points %
(Supplementary Material \ref{sec:numerical_issues}).

\begin{figure}[t]
\centering
\hrule
\includegraphics[width=0.95\textwidth]{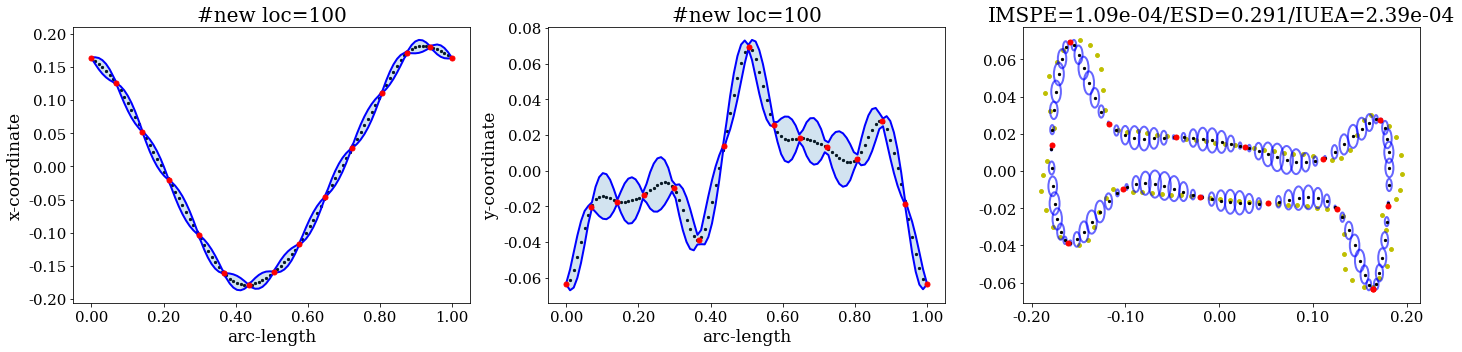}
\hrule
\caption{\textbf{Single-output GP model.} Baseline model fit for a bone curve from MPEG-7, %
observed at 15 
sample points
The first two columns show coordinate function fits, with prediction in black and 95\% pointwise prediction intervals represented as blue bands. The final column shows the overall curve fit with uncertainty ellipsoids.%
\newline \textbf{Figure conventions.} %
We display observed points in red, predictive mean in black, and points from the true curve in yellow. Uncertainty is depicted by blue ellipsoids at each prediction point, with principal axes driven by pointwise standard deviations for each coordinate. %
Metrics (IMPSE, ESD, IUEA) are reported at the top (Supplementary Material  \ref{sec:Metrics} for detailed definitions). Unless otherwise specified, all models are fit using periodic Matern 3/2 kernel with an additional constant kernel. %
}
\label{fig:sogp_fit}
\end{figure}

\subsection{Closedness and Periodic Kernels \protect\footnote{the $(j)$ superscript is dropped throughout this section for simplicity}}
\label{subsec:Closed}

As mentioned above, the challenge of using arc-length parameterization for closed curves is the following: if $\ell$ is the total length of a curve $\mathbf{f}$, then under an arc-length parameterization, the parameter value $s=\ell$ should also be mapped to the same point on the curve as that of parameter value $s=0$. %
Since the starting point can be arbitrarily chosen on $\mathcal{S}^1$ with any parameter $s$, we must take care to recognize that model fitting should be invariant to cyclic permutations of the rows of $\mathbf{Y}$, i.e., the ordering of the observed points.

One way of ensuring periodicity is to restrict our model to a single period.
An example of this is trigonometric regression \citep{eubank1990curve}, where 
trigonometric functions are used as basis functions to account for 
periodicity \citep{fisher1993statistical}. However, direct estimation of parameters associated with periodic basis functions can be difficult, %
due to an accompanying identifiability problem \citep{quinn1989estimating}. Furthermore, there is another issue with using these types of models for closed curves. When modeling univariate functional data, %
its range of parameter values can be bounded based on those corresponding to the (arbitrarily chosen) starting and ending points. The function can thus be more densely sampled within the domain without changing this range.
However, for a closed curve, parameter values are generally not known or provided. %
Since estimating a curve's total length depends on discretization, the range of the arc-length parameter will not be fixed; thus, standardization using a new total length estimate is required after every refinement. This means restricting estimation to within one period 
is impractical. Therefore, a closed curve parameterization should be monotonic (as with open curves) and also \emph{periodic}. 

Instead, we choose to handle periodicity by fitting a fully nonparametric model on the full closed manifold, adding periodic constraints. %
We do this by specifying a periodic covariance kernel \citep{duvenaud2014automatic} with appropriate constraints on hyperparameters, for the following reason. Suppose a closed curve $\mathbf{f}(s)$ is assumed to have
an extended domain $[0,\infty)$, such that the curve is re-traced
after every interval of length $\ell$. Simply coinciding starting and
ending points and restricting the model forces $\mathbf{f}(0)=\mathbf{f}(\ell)$, 
but there is no guarantee that the two identical pairs of points $\mathbf{f}(0),\mathbf{f}(\Delta)$ and $\mathbf{f}(\ell),\mathbf{f}(\ell+\Delta)$
are perfectly correlated for $\Delta>0$.
Ensuring this 
consistency in the dependence structure is important because the underlying curve $\mathbf{f}$ is unknown, meaning that its total length $\ell$ must be estimated. Based on this observation, $\mathcal{S}^{1}$ is a more appropriate modeling domain \citep{srivESA,olsen2018simultaneous}.

\begin{defn*}
(Periodic covariance kernels,   \citet{williams2006gaussian}, Sections 4.2 and B.1) The exponential periodic covariance kernel (with period $\tau$) is defined as:
\begin{align}
k_{(\sigma^{2},\rho,\tau)}(s_{i},s_{p})=\sigma^{2}\exp\left(-\frac{1}{\rho}\sin^{2}\left(\frac{\|s_{i}-s_{p}\|}{\tau/\pi}\right)\right)\,,\label{eq:periodic cov ker}
\end{align}
for $s_{i},s_{p} \in \mathcal{S}^1$ with hyperparameters $\sigma^{2}\in(0,\infty),\rho\in(0,\infty),\tau\in(0,\infty)$.
We require that 
there is an isometry between the chosen metric $\|\cdot\|$ and arc-length metric on $\mathcal{S}^1$. Other periodic covariance kernels (e.g., the Matern family) can be defined similarly.%
\label{defn:periodic_kernel}
\end{defn*}
The construction of warping by the sine function originates in circular statistics \citep{fisher1995statistical}. %
Unfortunately, $\tau$ is unidentifiable because of the presence of many local optima
stemming from the sine function: if $\hat{\tau}$ is a valid estimate of $\tau$, then estimates $\hat{\tau}+2\pi\cdot\mathbb{N}$ are also valid.
Instead of optimizing over $\tau$, we choose to empirically set $\tau$ close to the total length of the curve $\mathbf{f}$ by an estimate of the length of a piecewise-linear curve connecting the observed sample points as described in Figure \ref{fig:algorithm_xy_arc} (in Supplementary Material \ref{sec:arc-length Parameterization Algorithms}). 
As shown in Figure \ref{fig:single_shape_different_periodicities}, model fits are generally robust under this heuristic choice.
%

%
%

%
The length scale $\rho$ drives the strength of correlation between $s_{i},s_{p}$, while the period models recurrence of certain patterns. For a periodic kernel, the magnitude of correlation between two points 
depends on both the length scale and period. Lemma \ref{thm:bounds on periodic kernels} (in Supplementary Material \ref{sec:proof of main thm}) formally discusses the trade-off between these two kernel hyperparameters, but we demonstrate this with an example below. %

Figure \ref{fig:period_LS_fit} shows baseline model fits (i.e., Section \ref{subsec:Setup}) with a stationary periodic kernel to a bone. The fit is rather poor when there is no constraint on length scale and the period is 
estimated: while the predictive mean exhibits interpolating behavior, %
the period in the $x$-coordinate does not match that of the $y$-coordinate, resulting in a self-intersection on the bone's right side. Fixing the period (without length scale constraints) improves the overall fit. 
However, the length scale hyperparameter increases significantly in the $x$-coordinate, resulting in its uncertainty estimates being much smaller than those in the $y$-coordinate. If the length scale hyperparameter is constrained to be less than half of the estimated arc-length (i.e., the fixed period), %
we obtain an improved fit in which uncertainty is properly reflected in both the $x$- and $y$-coordinates. %

\begingroup
\renewcommand{\arraystretch}{0.5} %
\begin{figure}[t]
\centering

\begin{tabular}{cccc}
\toprule 
\multicolumn{1}{c}{(1)} & (2) & (3)\tabularnewline
\midrule 
\multicolumn{1}{c}{\includegraphics[width=0.3\textwidth]{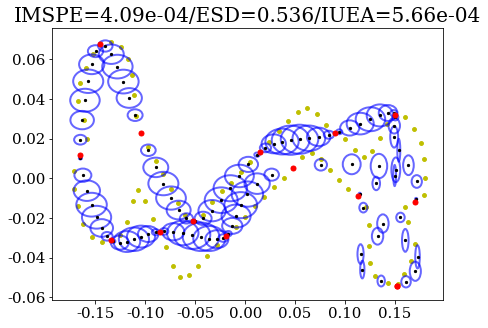}} & \includegraphics[width=0.3\textwidth]{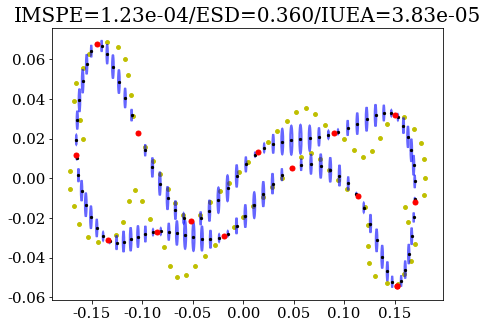} & \includegraphics[width=0.3\textwidth]{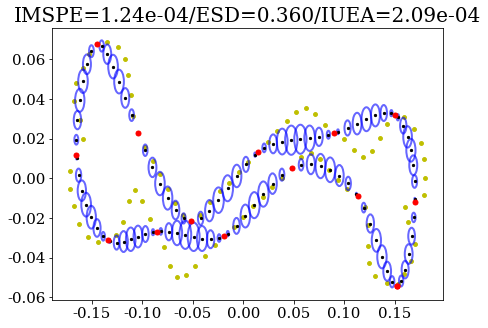}\tabularnewline
\midrule
\midrule 
LS $x=0.302$, LS $y=0.247$ & LS $x=6.80$, LS $y=0.247$ & LS $x=0.500$, LS $y=0.247$\tabularnewline
\midrule 
Per $x=0.907$, Per $y=1.00$ & Per $x=1.00$, Per $y=1.00$ & Per $x=1.00$, Per $y=1.00$\tabularnewline
\bottomrule
\end{tabular}
\caption{\textbf{Effect of length scale constraints with periodicity estimates on GP fits.} Baseline model fits for a bone curve from MPEG-7 observed at 15 points. Panel (1) %
does not constrain the length scale (LS) nor the period (Per).
Panel (2) fixes Per to the estimated arc-length, with LS unconstrained. Panel (3) fixes Per and constrains LS by $\rho\leq\tau/2$. Estimated values of Per and LS are reported for the 
coordinates $x,y$.
}%
\label{fig:period_LS_fit}
\end{figure}
\endgroup

Under mild assumptions of a closed domain, appropriate period constraints, and densely sample points on the closed curves, model consistency is guaranteed when coordinate functions of the underlying closed curve lie in a space induced by the periodic kernel. This is formally stated in Theorem \ref{thm:consistency thm} (Supplementary Material \ref{sec:proof of consistency thm}), based on results from 
\citet{koepernik2021consistency} and the approximation in Lemma \ref{thm:bounds on periodic kernels}.
Dense sampling schemes, where the distances between consecutive points converges to zero, are important for algorithmic and model consistency as the number of sample points tends to infinity. %
Finally, we note that GPs have the advantage of explicitly controlling the smoothness of the underlying curve explicitly through the choice of covariance kernel family \citep{cramer2013stationary}, which is crucial for subsequent analyses, e.g., in elastic shape analysis \citep{kurtek2012statistical}. %
Supplementary Material \ref{subsec:NonStat} discusses the ramifications of kernel choice on curve fitting under different sampling schemes and smoothness requirements. 

\subsection{Multiple-output Gaussian Processes for Multiple Curves}
\label{subsec:SingleCurveMOGP}

Even when using a periodic covariance kernel, fitting the baseline model from Section \ref{subsec:Setup} has two drawbacks: (1) it is unable to capture correlation between a fixed curve's two coordinate functions; and (2) when fit separately to a collection of $J$ curves, it implies curves have no structural similarity. It is natural to assume coordinates are correlated, since knowledge of the $x$-coordinate at arc-length parameter value $s$ can help in inferring the corresponding $y$-coordinate, particularly after accounting for dependence between different input parameter values. In addition, a collection of curves (perhaps representing similar objects) typically shares similar underlying structures, with any differences in details often occurring locally rather than globally. %
Our proposed multiple-output kernel accounts for the following sources of dependence in a collection of closed curves:
\begin{enumerate}
\item \emph{Within-coordinate dependence for a fixed curve:} This reflects the dependence between points $s_i^{(j)},s_p^{(j)} \in\mathcal{D}$ for a single fixed coordinate function $f_{d}^{(j)}$ of curve $j$, denoted by $\text{Cov}\left(f_{d}^{(j)}(s_i^{(j)}),f_{d}^{(j)}(s_p^{(j)})\right)$ for $d=1,2$.
For a closed curve, we choose to use the periodic kernel $k_{(\sigma^2,\rho,\tau)}^{(j)}(s_i^{(j)},s_p^{(j)})$ for the reasons described in the previous section.
\item \emph{Within-curve coordinate dependence for a fixed curve:} This reflects the dependence
between the coordinate functions of the curve $j$. When both coordinate functions $f_{d}^{(j)},f_{d'}^{(j)}$ are evaluated at a parameter value $s$, we denote this dependence by \newline $\text{Cov}\left(f_{d}^{(j)}(s),f_{d'}^{(j)}(s)\right)$, for $d,d' \in \{1,2\}$ with $d\neq d'$. 
\item \emph{Between-curve dependence:} This reflects the dependence between curves \newline $j,j' \in \{1,\ldots,J\}$ with $j\neq j'$, evaluated at a vector of parameter values $\bm{s}^{(j)},\bm{s}^{(j')}$. We denote this dependence by $\text{Cov}\left(\mathbf{f}^{(j)}(\bm{s}^{(j)}),\mathbf{f}^{(j')}(\bm{s}^{(j')})\right)$.
\end{enumerate}
Under this multi-level dependence structure, the first source can be directly incorporated into either single-output or multiple-output GPs. The latter two are only suitable for multiple-output GPs. 
If $J=1$, i.e., only one curve is observed, one can ignore the final source of dependence and incorporate the first two within a multiple-output model.

\textbf{Multiple-output Model for a Single Curve:} We first extend the single-output model of Section \ref{subsec:Setup} for a single curve $\mathbf{f}^{(j)}$ to incorporate within-curve coordinate dependence (2). Let $K_{\mathbf{f}^{(j)}}$ be a kernel such that for $s_i^{(j)},s_{p}^{(j)} \in \mathcal{D}$, we can express the covariance matrix between coordinate function outputs as:
\begin{equation}
K_{\mathbf{f}^{(j)}}(s_i^{(j)},s_p^{(j)})=
\begin{pmatrix}\text{Cov}\left(f_{1}^{(j)}(s_{i}^{(j)}),f_{1}^{(j)}(s_{p}^{(j)})\right) \ & \ \text{Cov}\left(f_{1}^{(j)}(s_{i}^{(j)}),f_{2}^{(j)}(s_{p}^{(j)})\right)\\
\text{Cov}\left(f_{2}^{(j)}(s_{i}^{(j)}),f_{1}^{(j)}(s_{p}^{(j)})\right) \ & \ \text{Cov}\left(f_{2}^{(j)}(s_{i}^{(j)}),f_{2}^{(j)}(s_{p}^{(j)})\right)
\end{pmatrix}.
\label{eq:MOGPmatrix}
\end{equation}
Thus, $K_{\mathbf{f}^{(j)}}$ is identified by specifying
a general expression for $\text{Cov}\left(f_{d}^{(j)}(s_{i}^{(j)}),f_{d'}^{(j)}(s_{p}^{(j)})\right)$, with $d,d'\in\{1,2\}$. A common approach is to specify a \emph{separable kernel}: define $K_{\mathbf{f}^{(j)}}=k^{(j)}\otimes k_D^{(j)}$ 
as the Kronecker product of the scalar kernels $k^{(j)}$ and $k_{D}^{(j)}$, where $k^{(j)}$ is the input kernel (i.e., chosen to be a periodic stationary kernel as in Section \ref{subsec:Closed}) and $k_D^{(j)}$ is the within-curve coordinate kernel for curve $j$. This means
$\text{Cov}\left(f_{d}^{(j)}(s_{i}^{(j)}),f_{d'}^{(j)}(s_{p}^{(j)})\right)=k^{(j)}(s_{i}^{(j)},s_{p}^{(j)})k_{D}^{(j)}(d,d')$.
If both $k^{(j)}$ and $k_D^{(j)}$ are valid kernels (i.e., symmetric and positive semi-definite), then their product is also a valid kernel \citep{williams2006gaussian}. 

In the baseline model of Section \ref{subsec:Setup} (with uncorrelated coordinate functions), \eqref{eq:MOGPmatrix} is assumed to have off-diagonal entries equal to zero, meaning that  $k_D^{(j)}(d,d')=\delta(d,d')$ for Kronecker delta $\delta$. 
However, suitable modifications of the within-curve coordinate kernel $k_D^{(j)}$ can flexibly allow correlation between coordinates. 
We note that a separable kernel assumes that similarity for multiple outputs can be characterized independently across every dimension. In the setting of closed curves, this means that we are implicitly assuming that $k_D^{(j)}$ remains the same across all possible pairs of inputs $s_i^{(j)},s_{p}^{(j)} \in \mathcal{D}$. This assumption nicely decomposes correlation into separate components attributed between inputs and between coordinate functions. Unlike the periodic requirement for within-coordinate dependence, this hierarchy can be extended to higher dimensions. Alternatively, one could construct more complex interactions between kernels $k^{(j)}$ and $k_D^{(j)}$.

Note that the matrix in \eqref{eq:MOGPmatrix} can be written as $k^{(j)}(s_{i}^{(j)},s_{p}^{(j)})\cdot D^{(j)}$ where $D^{(j)}$ is a $2\times2$ symmetric and positive semi-definite matrix such that $D^{(j)}_{d,d'}=k_{D}^{(j)}(d,d')$. The matrix $D^{(j)}$ is often referred to as a \emph{coregionalization matrix}. Estimation of $D^{(j)}$ means estimating its entries subject to the constraint of symmetry and positive semi-definiteness, and quantifies the dependence accounted for between a single curve's coordinate functions. The above construction of product kernel follows from \citet{bonilla2007multi}.
For the remainder of this paper, we focus on the product kernel. %
To ensure positive semidefiniteness, we follow  the Python module \texttt{GPFlow2} \citep{GPflow2017,GPflow2020multioutput} convention by parameterizing the coregionalization matrix as $D^{(j)}=W^{(j)}(W^{(j)})^T + \text{diag}(\boldsymbol{\kappa}^{(j)})$, where $W^{(j)}$ is a $2\times r$ matrix with $\text{rank}(W^{(j)})=r\leq 2$ and $\boldsymbol{\kappa}^{(j)} \in \mathbb{R}^2$. 
In general, increasing the rank yields more flexible matrices $D^{(j)}$.

%
%
%
%
%
%
%
%
%
%
%
%
%
%

%

%
%

%
%

%
%
%
%
%
%
%
%
%
%
%
%
%
%
%
%
%
%
%
%
%
%
%
%
%
%
%
%

%
%
%

\textbf{Multiple-output Model for Multiple Curves:} In order to incorporate between-curve dependence (3), we borrow the coregionalization specification used 
to incorporate within-curve coordinate dependence for a fixed curve to also exploit structural similarities across curves. 
This between-curve dependence is conditional on existing dependence structures between inputs (parameter values) and curve-specific outputs (dimensions) %
 after preprocessing (Supplementary Material \ref{subsec:PreProc}). %

To specify the full Gaussian process model which incorporates this multi-level dependence structure across all $J$ curves, we consider the following joint model specified on stacked vectorizations of each curve's sample points:
\begin{gather}
\text{vec}\left( \mathbf{Y} \right) =
\left(\begin{array}{c}
\text{vec}\left( \mathbf{Y}^{(1)} \right)\\
\vdots\\
\text{vec}\left( \mathbf{Y}^{(J)} \right)
\end{array}\right) =\left(\begin{array}{c}
\mathbf{f}^{(1)}(\bm{s}^{(1)})\\
\vdots\\
\mathbf{f}^{(J)}(\bm{s}^{(J)})
\end{array}\right)+\left(\begin{array}{c}
\bm{\epsilon}^{(1)}\\
\vdots\\
\bm{\epsilon}^{(J)}
\end{array}\right)\in\mathbb{R}^{\sum_{j=1}^J 2n_{j}} \\
\mathbf{f}^{(j)}(\mathbf{s}^{(j)}) \coloneqq \left( f_{1}^{(j)}(s_{1}), f_{2}^{(j)}(s_{1}), \hdots, f_{1}^{(j)}(s_{n_{j}}), f_{2}^{(j)}(s_{n_{j}}) \right)^T \\
\bm{\epsilon}^{(j)} \coloneqq \left( \epsilon_{11}^{(j)}, \epsilon_{12}^{(j)},\hdots,\epsilon_{n_{j}1}^{(j)},\epsilon_{n_{j}2}^{(j)}\right)^T \,.
\label{eq:MOGP_mult_curves_vec}
\end{gather}

%
%
%
%
%
%
%
%
%
%
%
%
%

\begingroup
\renewcommand{\arraystretch}{0.5} %
\begin{figure}[t!]
\centering 
\begin{tabular}{{@{}c@{}|@{}c@{}@{}c@{}@{}c@{}}}
\hline
& Camel 1 & Camel 2 & Camel 3 \\
\hline
\centered{(a)}&\centered{\includegraphics[width=0.28\textwidth]{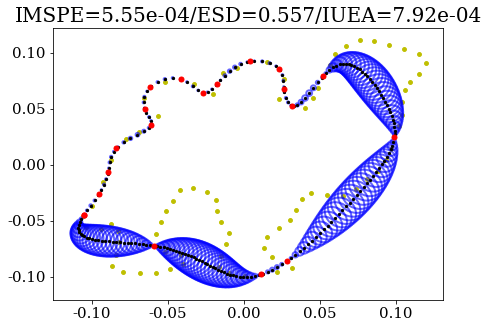}}&\centered{\includegraphics[width=0.28\textwidth]{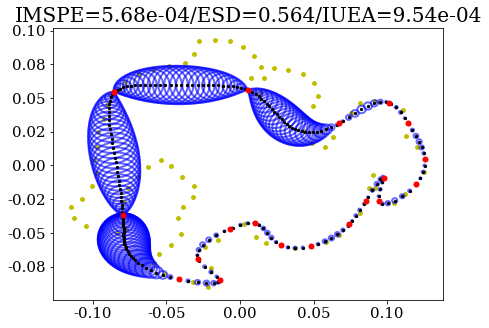}}&\centered{\includegraphics[width=0.28\textwidth]{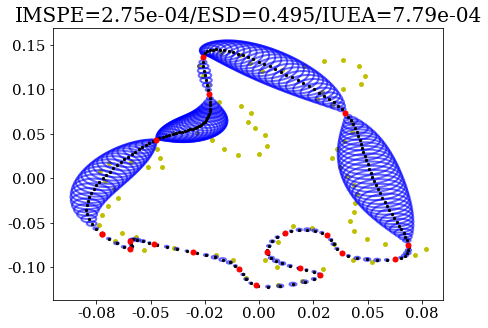}}\\
\hline
\centered{(b)}&\centered{\includegraphics[width=0.28\textwidth]{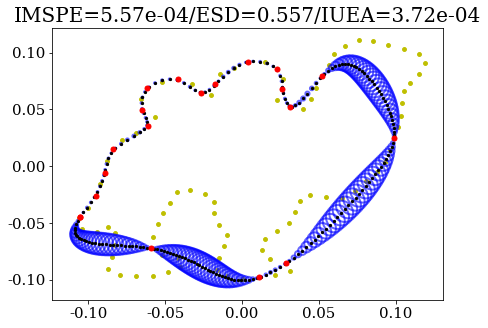}}&\centered{\includegraphics[width=0.28\textwidth]{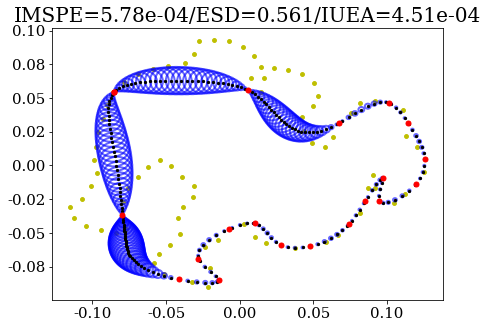}}&\centered{\includegraphics[width=0.28\textwidth]{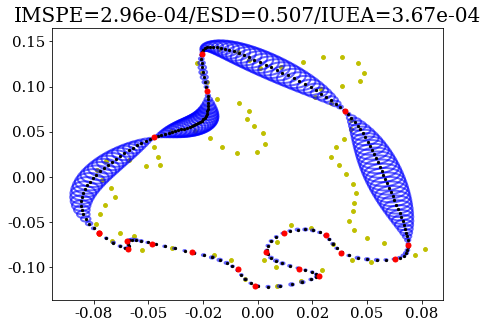}}\\
\hline
\centered{(c)}&\centered{\includegraphics[width=0.28\textwidth]{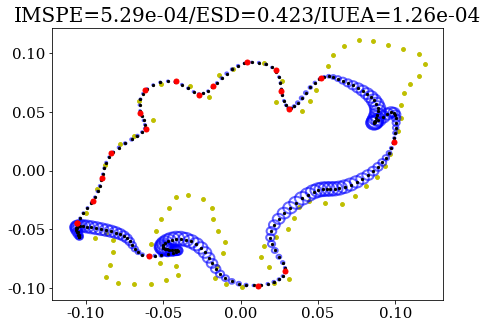}}&\centered{\includegraphics[width=0.28\textwidth]{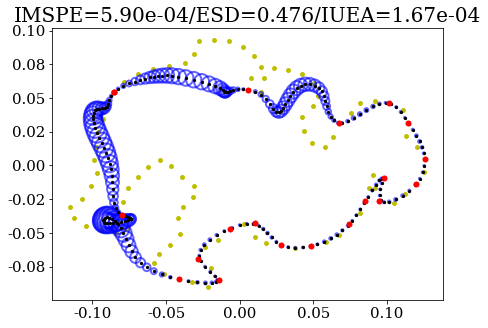}}&\centered{\includegraphics[width=0.28\textwidth]{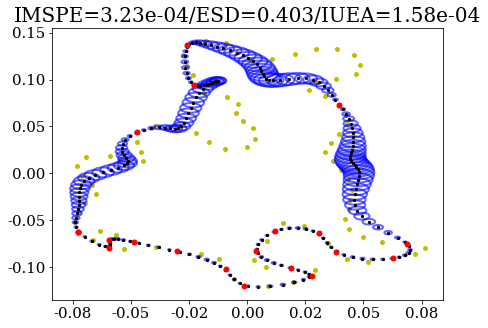}}\\
\hline
\centered{(d)}&\centered{\includegraphics[width=0.28\textwidth]{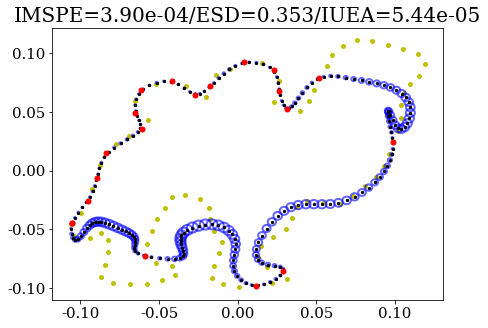}}&\centered{\includegraphics[width=0.28\textwidth]{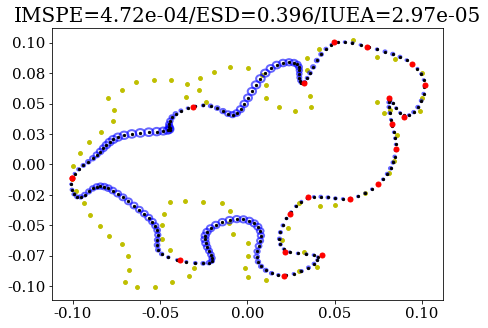}}&\centered{\includegraphics[width=0.28\textwidth]{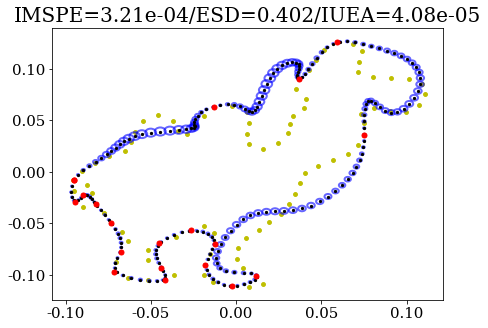}}\\
\hline
\end{tabular}
 \caption{\textbf{Comparisons between baseline model and joint multiple-output GP models for multiple curves.} Three camel curves from MPEG-7, sampled at 20 points. In different rows, we show (a) baseline model fits for each curve, (b) separate multiple-output GP fits for each curve, and a (c) jointly fit multiple-output GP, all without preprocessing. Row (d) shows the fit of model (c), with preprocessing.}%
\label{fig:multishape_comparison_camel} 
\end{figure}
\endgroup

For this joint model across curves, we also choose to specify a separable kernel: for scalar kernels $k$, $k_D$, and $k_C$, let $K=k \otimes k_D \otimes k_C$ have elements
$\text{Cov}\left( f^{(j)}_{d}(s_i^{(j)}),f^{(j')}_{d'}(s_p^{(j')}) \right) = k(s_i^{(j)},s_p^{(j')})k_D(d,d')k_C(j,j')$.
The within-curve coordinate dependence kernel $k_D$ is specified %
as above, with coregionalization matrix $D$ that has entries $D_{d,d'}=k_D(d,d')$, constrained to be symmetric and positive semi-definite. Due to the assumption of separability, note that 
$k_D$ is independent of the curve index $j$, which fixes the within-curve coordinate dependence across all curves. %
To incorporate between-curve dependence, we specify a between-curve scalar kernel $k_C$ such that its evaluation $k_C(j,j')$ measures the similarity between curves $\mathbf{f}^{(j)}$ and $\mathbf{f}^{(j')}$. Similar to $k_D$, a simple choice is to use the intrinsic coregionalization model, defining a $J \times J$-dimensional coregionalization matrix $C$ such that $C_{j,j'}=k_C(j,j')$. Fitting this model amounts to estimating hyperparameters of kernels $k,k_D,k_C$. %
Based on the three-dimensional kernel $K$, the corresponding covariance matrix of dimension $\sum_{j=1}^J 2n_j$  for all observed sample points across all $J$ curves can be computed.%

Figure \ref{fig:multishape_comparison_camel} shows a comparison of the baseline model, multiple-output GP, and joint multiple-output GP fits, where each of the three camels is cluster sampled in a different region: the first around the humps, the second around the camel's front, and the third around its rear legs and tail. Curves are re-scaled to unit length and zero-centered; the bottom row also rotates the last two camels to match the first camel prior to model fitting, whereas the other panels do not have preprocessing (Supplementary Material \ref{subsec:PreProc}). 

Note the reduction in predictive uncertainty obtained in fitting multiple-output models compared to the baseline model. Additionally,  %
we obtain improved mean curve prediction under joint modeling of the three curves, as each curve has borrowed detailed information from other curves in different regions of the camel. The rotation preprocessing actually impacts fits quite differently: %
the first and second camels have more distinct rear legs, and apparent self-intersections in the predictive mean for the second camel disappear when properly rotationally aligned.
The aligned multiple-output GP recovers the original boundaries of all three camels much better that the other models, both visually and quantitatively. 

\section{Applications}
\label{sec:Future}

\subsection{Robust Representative Shape Summaries}
\label{subsec:CurveRec}

Consider a scenario where 
observed points are sparse samples of the outlines of three leaves from the Flavia dataset \citep{wu2007leaf}, and the goal is to compute a representative shape summary.
A standard approach in elastic shape analysis is to compute the Karcher mean, found by minimizing a sum of squared distances in the elastic shape space. Distance calculations involve shape registration, which can perform unfavorably in sparse sampling regimes and in the presence of noise (Supplementary Material  \ref{subsec:ElasticShapeReg}). Even if leaves are arbitrarily resampled to 101 points, Karcher mean estimates can be quite sensitive to sparse samples, clustered sampling schemes, and noise,
as illustrated by Figure \ref{fig:curve_rep}.%

\begingroup
\renewcommand{\arraystretch}{0.5} %
\begin{figure}[t]
\centering \begin{adjustbox}{center}
\begin{tabular}{c|c|c|c|c}
\hline
 &(a) No noise, $10$ & (b) No noise, $20$ & (c) No noise, $50$ & (d) Noise, $50$\\
\hline
&\includegraphics[width=0.2\textwidth]{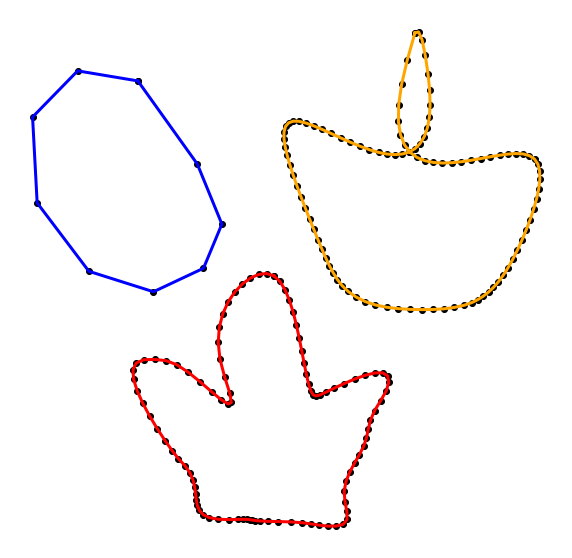}&\includegraphics[width=0.2\textwidth]{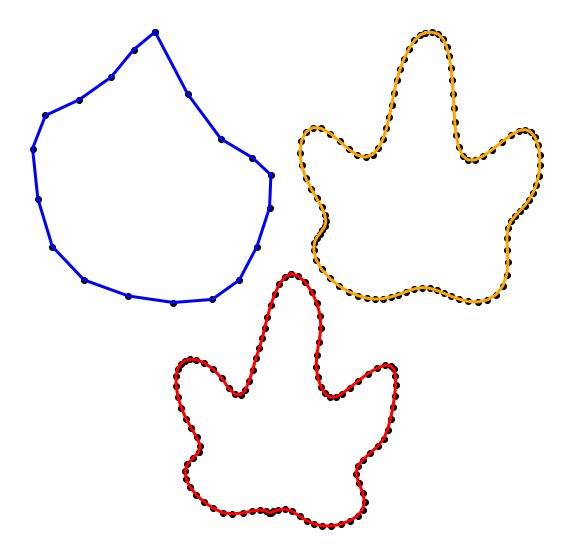}&\includegraphics[width=0.2\textwidth]{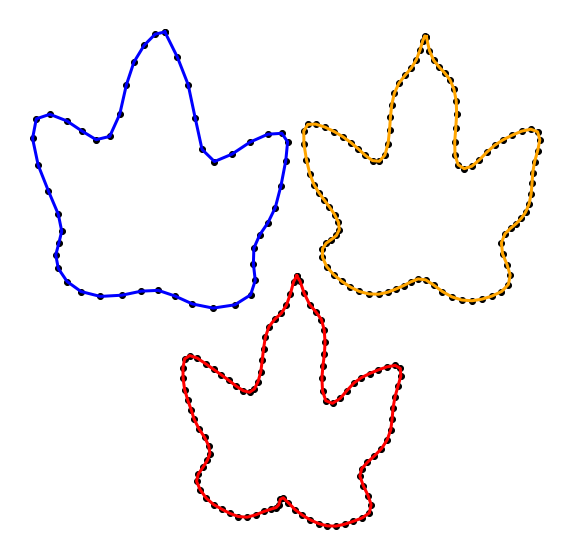}&\includegraphics[width=0.2\textwidth]{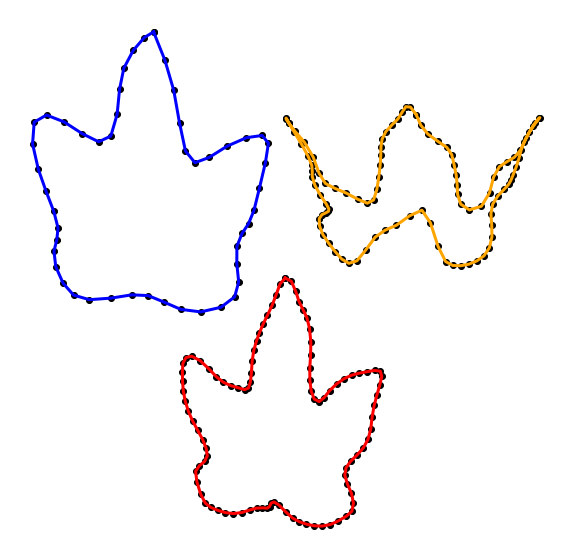}\\
\hline
\rowcolor{blue!30}
(I) & $3.38 \times 10^{-2}$ & $2.39 \times 10^{-2}$ & $1.41 \times 10^{-2}$ & $1.43 \times 10^{-2}$ \\
\hline
\rowcolor{orange!30}
(II) & $1.90 \times 10^{-2}$ & $1.35 \times 10^{-2}$ & $9.32 \times 10^{-3}$ & $2.62 \times 10^{-2}$ \\
\hline
\rowcolor{red!30}
(III)  & $4.86 \times 10^{-3}$ & $1.30 \times 10^{-2}$ & $2.03 \times 10^{-3}$ & $2.37 \times 10^{-3}$ \\
\hline
\end{tabular}
\end{adjustbox}
 \caption{\textbf{Comparison of GW distances to ``ground-truth'' mean shapes from GP modeling and Karcher mean computation.} 
Panels (a)-(d) show mean summary comparisons for 3 Flavia leaf curves, observed noise-free at 10, 20, and 50 sample points (a-c) and with Gaussian noise (standard deviation $0.007$) at 50 sample points (d). Each panel shows the Karcher mean computed from observed sample points (blue, (I)) and from interpolative re-sampling to 101 points (orange, (II)), along with pointwise predictive mean from a multiple-output GP fit (red, (III)). The GW distance between corresponding colored curve and the Karcher mean computed using 101 sample points. %
}
\label{fig:curve_rep} 
\end{figure}
\endgroup

An alternative is to fit a multiple-output GP to the collection of curves, and then take pointwise averages of the predictive means across curves. This ensures that the representative curve is also closed, and modeling between-curve dependence learns the relationships across curves at varying parameter values. Figure \ref{fig:curve_rep} shows that this pointwise predictive mean characterizes the representative shape of the three Flavia leaves %
more consistently under varied sampling regimes. Increasing the number of observed sample points (panels (b) and (c)) can mitigate some of these issues accompanying the raw and re-sampled Karcher means, where pointwise predictive means look more similar to the Karcher mean estimated using the full 101 points; however, the presence of noise can exacerbate issues, as shown in panel (d), while the pointwise GP predictive mean is robust to noise. 
The bottom of the figure shows 
GW distances %
between each representative shape summary and the Karcher mean computed using all 101 points. Note that the pointwise GP predictive mean outperforms other alternatives here.

\subsection{Landmark Detection}
\label{subsec:lmk}
Another application within shape analysis that crucially relies on valid uncertainty quantification is %
landmark detection. %
Theorem \ref{thm:consistency thm} guarantees model consistency when curves are densely sampled. However, dense samples are not always practical for modeling or computing purposes. %
Thus, detecting a small set of observed sample points (referred to as landmarks) to adequately %
compress the original curve in a principled way is important.  
As outlined below, this can be done either (i) simultaneously (jointly estimate a landmark set, e.g., \cite{strait2019lmk}), or (ii) sequentially (conditionally estimate the next landmark given a current landmark set, e.g., \cite{gao2019gaussian}). Our proposed GP model can use uncertainty estimates to handle either task.
To simultaneously estimate $p$ landmarks separately on a collection of multiple curves, %
we propose the following simple procedure: (1) repeatedly generate $p$ random parameter values, denoted $\bm{s}$, common to all curves; (2) fit the multiple-output GP model to the Euclidean coordinates corresponding to $\bm{s}$; (3) compare model fits to the original, densely sampled points using performance metrics (Supplementary Material \ref{sec:Metrics}). For instance, we may want the landmark set $\bm{s}$ which minimizes IMSPE, averaged over curves.
Similar to \cite{strait2019lmk}, if an estimate $\hat{p}$ of the number of landmarks is desired, the chosen performance metric can be evaluated over a set of pre-specified values of $p$. Then, $\hat{p}$ is selected either qualitatively, as the smallest value where the minimum IMSPE begins to stabilize, or quantitatively using, e.g., the gap statistic \citep{tibshirani2001gap}.

Figure \ref{fig:sim_lmk_leaf} demonstrates this simultaneous landmarking procedure on three Flavia leaves, assumed to be densely sampled. For $p \in \{3,\ldots,20\}$ landmarks, 1000 sets of landmarks $\bm{s}$ are randomly sampled without replacement from the densely-sampled curves, and the joint multiple-output GP model is fit. Each fit is evaluated based on average IMSPE and IUEA across the three curves; the random landmark configuration which minimizes these quantities for each $p$ is shown in the top left panel. IMSPE stabilizes around $\hat{p}=10$ landmarks, whereas the IUEA more steadily decreases. The other panels show the joint fit for the optimally chosen $\hat{p}=10$ landmarks under the IMSPE criteria; note that the fitted curves here provide a fairly adequate approximation of the original curves, supporting our preference for observed sample points to be equally spaced along the curve domain.

\begin{figure}[t]
\centering 
\begin{tabular}{cc}
\hline
\includegraphics[width=0.35\textwidth]{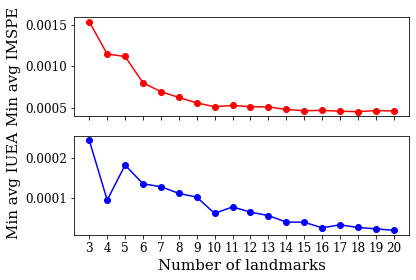}&\includegraphics[width=0.35\textwidth]{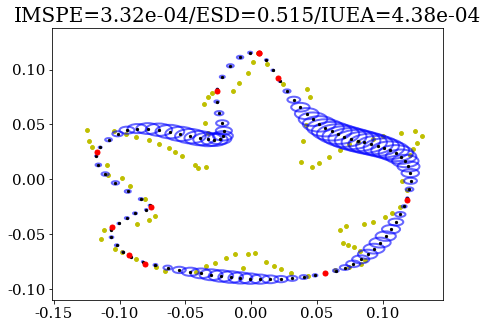}\\
\includegraphics[width=0.35\textwidth]{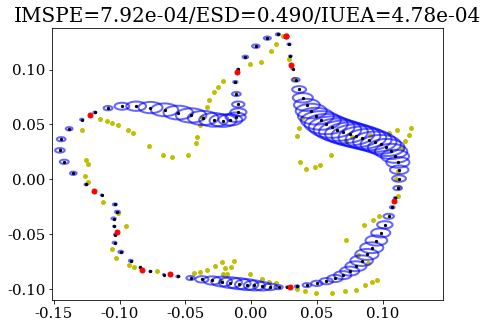}&\includegraphics[width=0.35\textwidth]{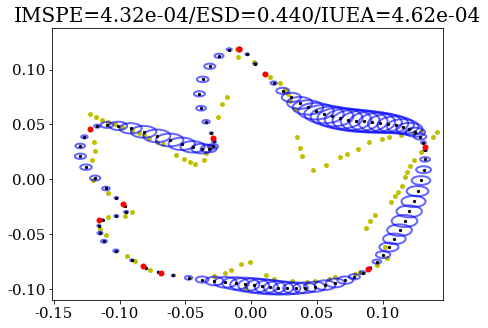}\\
\hline
\end{tabular}
\caption{\textbf{Simultaneous landmarking for multiple curves.} Three leaves from the same class of the Flavia dataset (top right and bottom panels), modeled using jointly fit multiple-output GP based on $\hat{p}=10$ landmarks. The top left panel shows the minimum average IMSPE (red) and IUEA (blue) as a function of the locations of $p$ landmarks.}
\label{fig:sim_lmk_leaf} 
\end{figure}

\begingroup
\renewcommand{\arraystretch}{0.5} %
\begin{figure}[t!]
\centering 
\begin{tabular}{ccc}
\hline
Single-output GP & Separate multiple-output GP & Joint multiple-output GP\\
\hline
\includegraphics[width=0.296\textwidth]{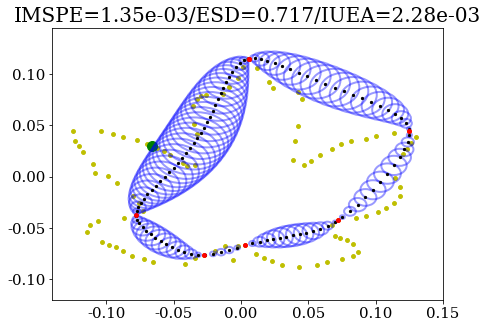}&\includegraphics[width=0.3\textwidth]{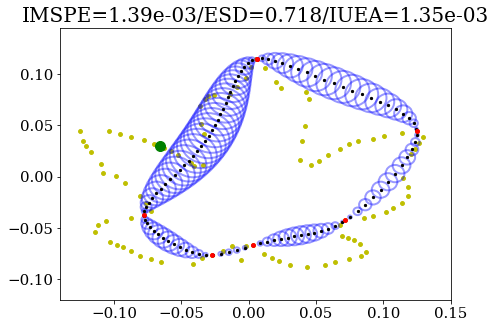}&\includegraphics[width=0.3\textwidth]{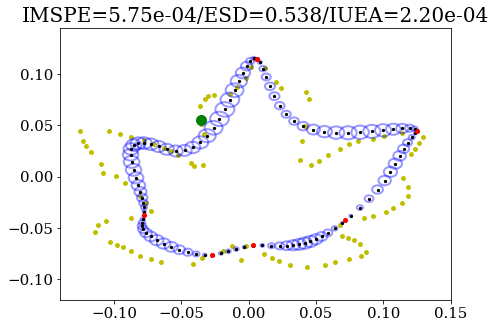}\\
\includegraphics[width=0.296\textwidth]{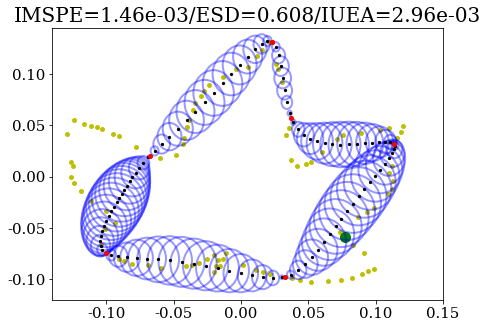}&\includegraphics[width=0.3\textwidth]{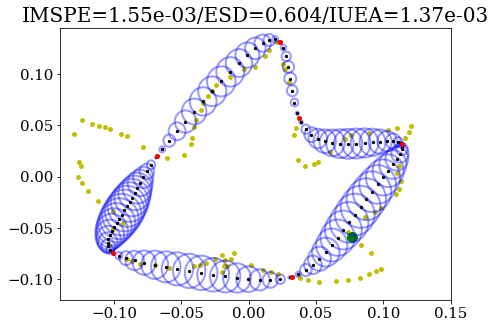}&\includegraphics[width=0.3\textwidth]{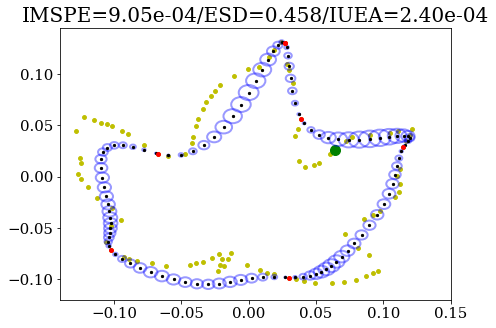}\\
\includegraphics[width=0.296\textwidth]{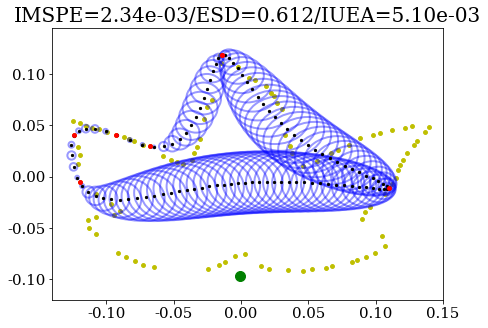}&\includegraphics[width=0.3\textwidth]{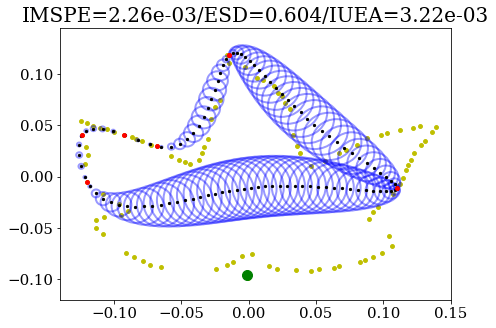}&\includegraphics[width=0.3\textwidth]{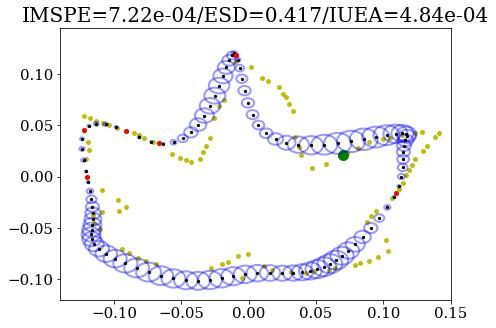}\\
\hline
\end{tabular}
\caption{\textbf{Iterative landmarking comparison 
between various GP models for multiple curves.} Three leaves from the same class of the Flavia dataset, modeled using single-output GPs (left column), separate multiple-output GPs for each curve (middle) and a jointly fit multiple-output GP (right). Each leaf is observed at 7 sample points, randomly selected from the true curve. The next landmark, as suggested by Equation \eqref{eq:lmk_opt} with $\lambda=0.5$, is illustrated in green on the true curve.  
}
\label{fig:lmk_leaf_comp} 
\end{figure}
\endgroup

We also outline a sequential procedure for landmarking, inspired by \cite{gao2019gaussian}, as follows: (1) fit a GP model to the curve(s); (2) obtain the estimated %
prediction variance matrix for both coordinates of a single curve, $\tilde{\Sigma}_p(s^*)$, with each coordinate's predictive variances $\tilde{\sigma}_{p1}(s^*), \tilde{\sigma}_{p2}(s^*)$ on the diagonal; (3) set the next landmark $s_{p+1}$ as:
\begin{equation}
s_{p+1} = \underset{s^* \in \mathcal{D}}{\text{argmax}} \ \left(\lambda\tilde{\sigma}_{p1}(s^*) + (1-\lambda)\tilde{\sigma}_{p2}(s^*)\right) \,,
\label{eq:lmk_opt}
\end{equation}
where $\lambda \in [0,1]$ controls the weight  %
of the first coordinate's predictive variance.%

Figure \ref{fig:lmk_leaf_comp} compares the sequential landmarking procedure under different GP models for selecting the next landmark on three Flavia leaves, each sampled at 7 random points along the true curve. 
GPs are fit under the single-output, separate multiple-output, and joint multiple-output kernel formulations. Once fit, the next landmark is identified as the point on the true curve which solves \eqref{eq:lmk_opt}. 
In the first two columns, the landmarks are in different locations across curves as a result of between-curve independence. %
However, the joint multiple-output model (in the last column) %
communicates information across the three leaves through the coregionalization kernel, by exploiting regions in which each individual leaf has more information in fitting; as a result, curve fits more closely resemble each individual leaf.
Eventually, landmarks will target those individual curve regions which are less sampled, as total uncertainties are driven down in other regions. %

Note that both sequential and simultaneous landmark detection appear to prefer landmark configurations that resemble a uniform sampling. In a sense, the uniform sampling is not only the default choice in a dense sampling to ensure both model consistency (Theorem \ref{thm:consistency thm}) and arc-length parameterization consistency (Theorem \ref{thm:con_xy_arc}), but is also what GP-based simultaneous and sequential landmarking will tend towards.

\subsection{Tooth Reconstruction}
\label{subsec:tooth}

The Tufts Dental Database\footnote{\url{http://tdd.ece.tufts.edu/}} \citep{tuftsteeth} contains 1000 panoramic dental X-ray images, along with additional expert annotations of abnormal teeth and other miscellaneous information. The top panel of Figure \ref{fig:teeth} shows X-ray images for three patients, selected because they have tooth numbers 31 and 32, outlined in the bottom left of each image. Use of a boundary estimation procedure to extract the outlines of these teeth, e.g., active contours \citep{Snake}, does not satisfactorily yield separate contours for each teeth due to the nature of these binary images \citep{luo2019combining}. In particular, there is a significant gap between these two teeth in patient 1, whereas these teeth appear ``joined together'' in patients 2 and 3. We want to estimate outlines of the individual teeth,
which can be viewed through the task of \emph{curve reconstruction}: use model prediction to complete curves in a region for which data is unavailable. 

\begingroup
\renewcommand{\arraystretch}{0.5} %
\begin{figure}[t!]
\centering \begin{adjustbox}{center}
\begin{tabular}{ccc}
\hline
Patient 1 & Patient 2 & Patient 3\\
\hline
\includegraphics[width=0.25\textwidth]{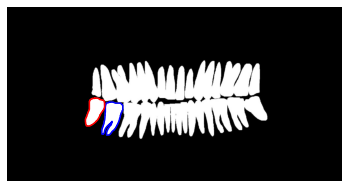}&\includegraphics[width=0.25\textwidth]{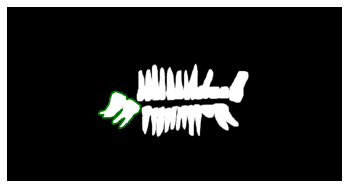}&\includegraphics[width=0.25\textwidth]{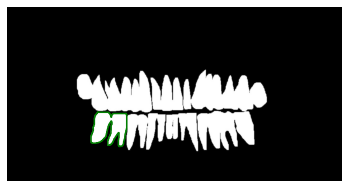}\\
\includegraphics[width=0.25\textwidth]{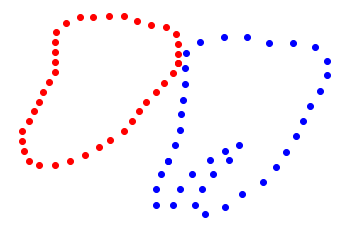}&\includegraphics[width=0.25\textwidth]{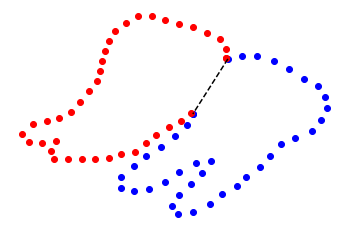}&\includegraphics[width=0.25\textwidth]{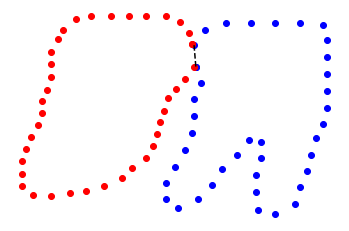}\\
\hline
\includegraphics[width=0.25\textwidth]{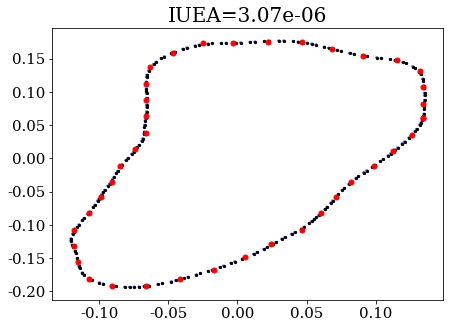}&\includegraphics[width=0.25\textwidth]{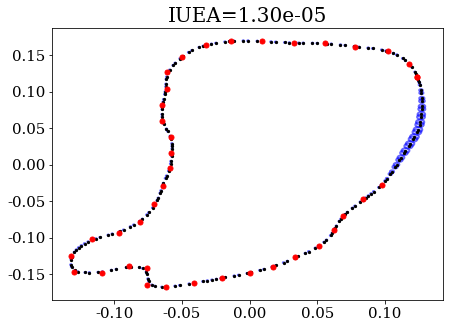}&\includegraphics[width=0.255\textwidth]{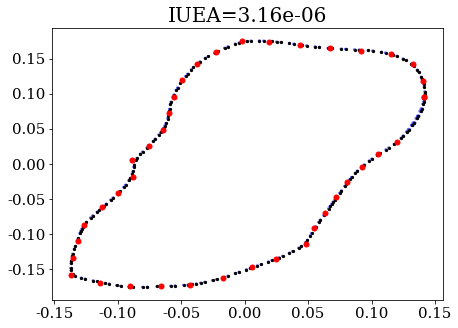}\\
\includegraphics[width=0.25\textwidth]{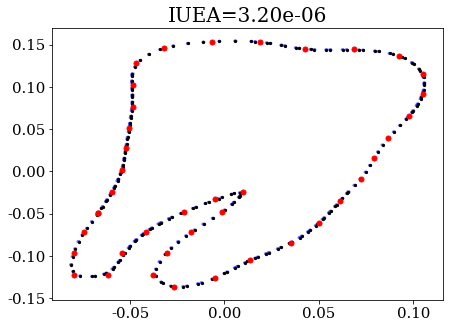}&\includegraphics[width=0.25\textwidth]{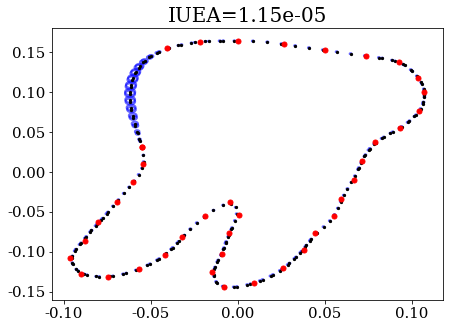}&\includegraphics[width=0.26\textwidth]{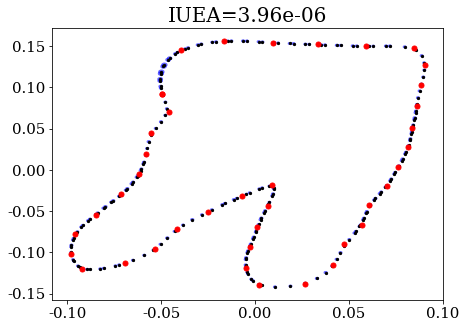}\\
\hline
\end{tabular}
\end{adjustbox}
\caption{\textbf{Joint multiple-output GP model for closed teeth outline prediction.} Reconstruction of separate tooth boundaries for three patients from the teeth dataset. Top row shows initial patient x-ray images with the relevant teeth outlined; second row shows the data (40 observed sample points each) 
}
\label{fig:teeth} 
\end{figure}
\endgroup

We can use the proposed joint multiple-output GP model to achieve this goal of separate tooth reconstruction. In particular, we identify the two teeth (31 and 32) 
separately in patient 1 (colored differently in the second row of Figure \ref{fig:teeth}). For patients 2 and 3, we manually split points on the green outlines (a single closed curve for each patient) into two separate sets at points of high curvature.
These separate curves are treated as closed by connecting the split points to each other, as shown by the dashed lines in the figure, to follow the practice in \citet{matthews21teeth}. Since there is no observed data in the region where we expect the boundary between teeth, we expect increased uncertainty estimates. 

The collection of red curves (the left-most teeth for each patient) is input to fit one multiple-output GP, and similarly the collection of blue curves (the tooth directly to the right) is the input for a second multiple-output GP, both after preprocessing. Predictive mean and uncertainty ellipsoids are shown in the third and fourth rows of Figure \ref{fig:teeth}, respectively. 
Note that the predictive mean is nonlinear in the missing boundary between teeth for patients 2 and 3, and the uncertainty is properly characterized due to the borrowing of structural information about the tooth in that region across patients. 

\subsection{Modeling Sub-populations}
\label{subsec:ClusteringMethods}

\begingroup
\renewcommand{\arraystretch}{0.5} %
\begin{figure}[t!]
\centering 
\begin{tabular}{c|c|c}
\hline
& (a) & (b)\\
\hline
\includegraphics[width=0.3\textwidth]{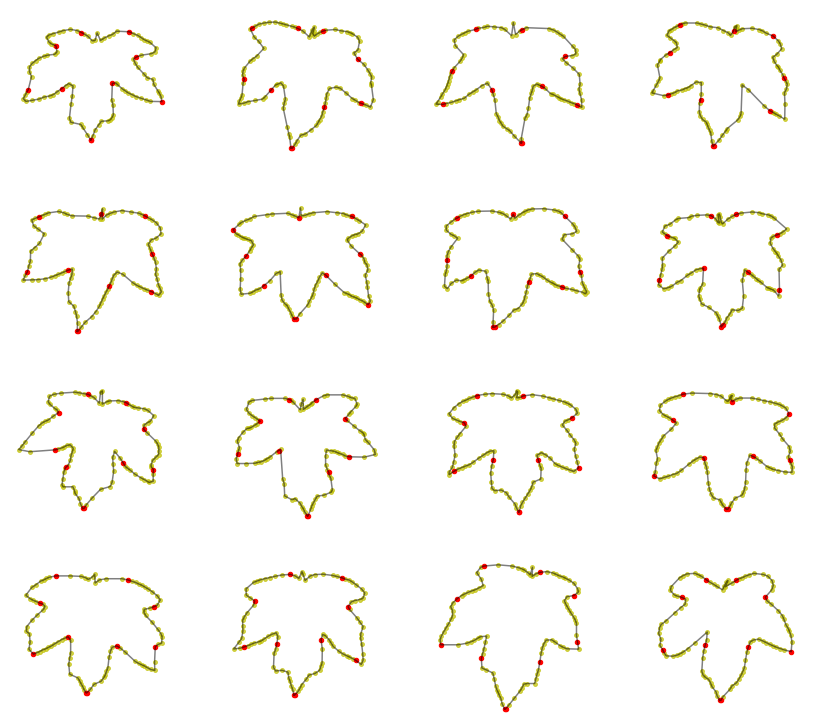}&\includegraphics[width=0.3\textwidth]{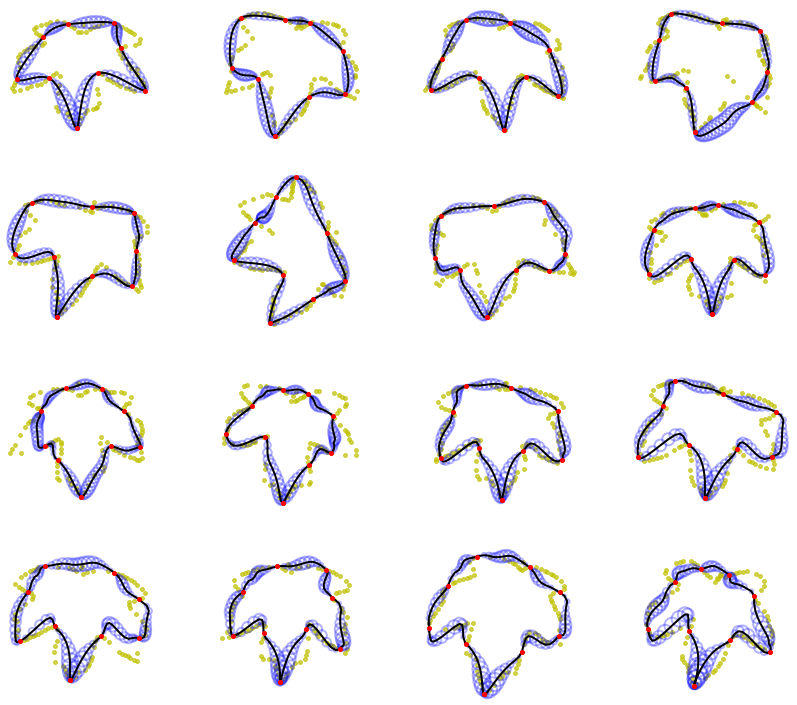}&\includegraphics[width=0.3\textwidth]{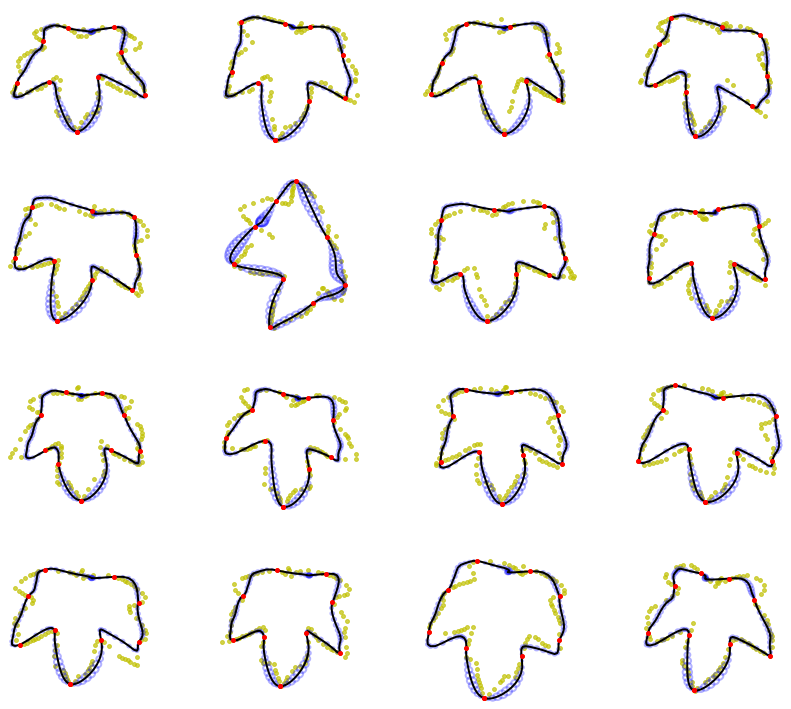}\\
\hline
Class 1 mean IMSPE & $6.62\times 10^{-4}$ & $5.41\times 10^{-4}$\\
Class 1 mean ESD & 0.474 & 0.423 \\
Class 1 mean IUEA & $2.15 \times 10^{-4}$ & $5.33 \times 10^{-5}$\\
\hline
\hline
\includegraphics[width=0.3\textwidth]{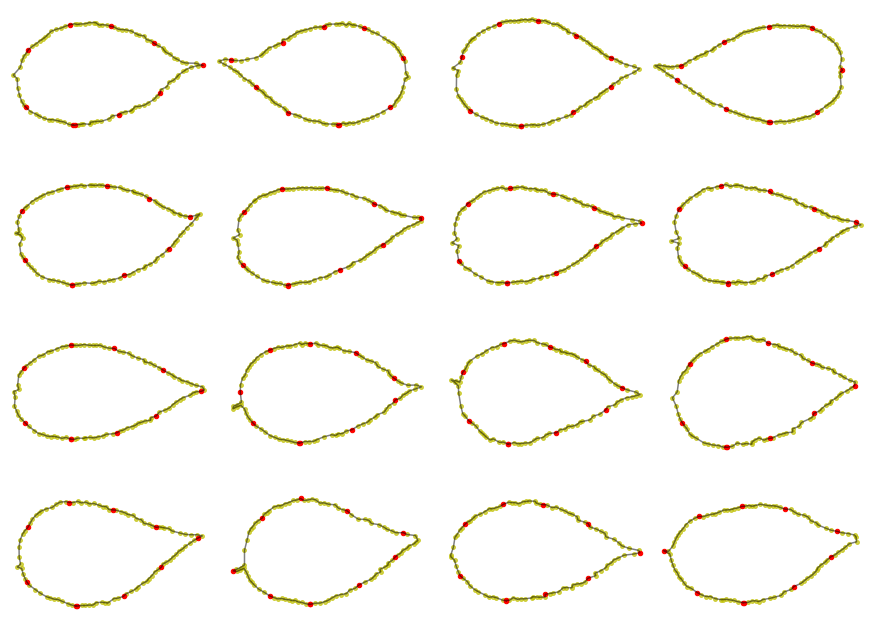}&\includegraphics[width=0.3\textwidth]{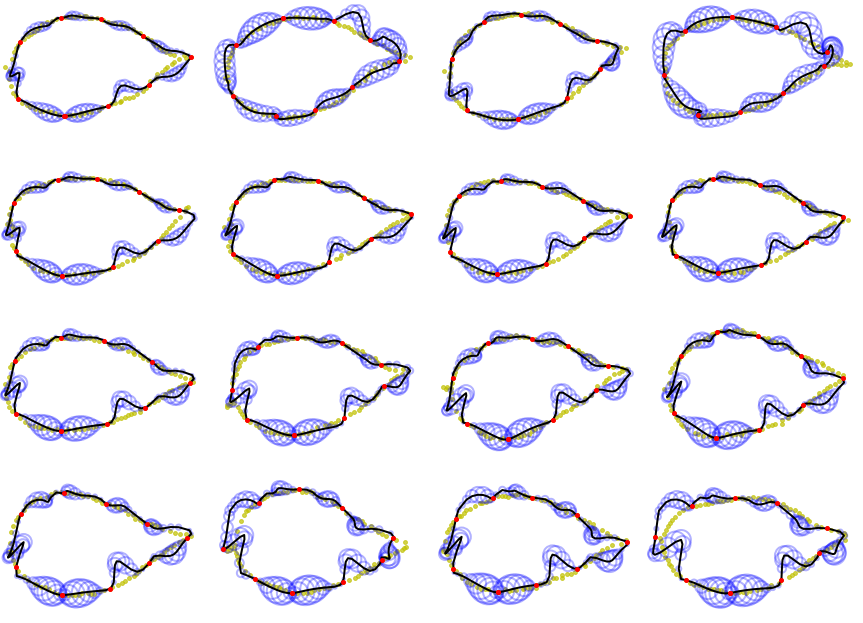}&\includegraphics[width=0.3\textwidth]{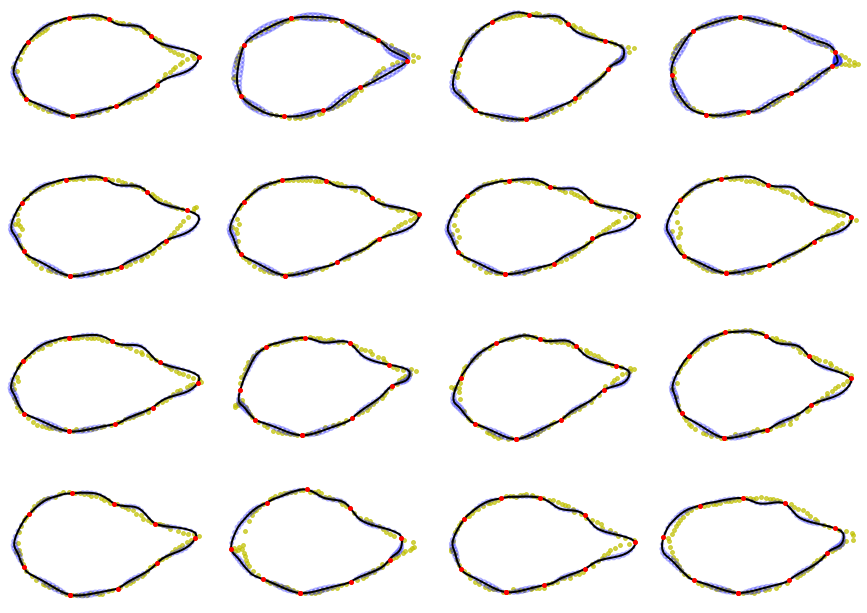}\\
\hline
Class 2 mean IMSPE & $6.01 \times 10^{-4}$ & $3.41 \times 10^{-4}$\\
Class 2 mean ESD & 0.407 & 0.249 \\
Class 2 mean IUEA & $5.32 \times 10^{-4}$ & $5.11 \times 10^{-5}$\\
\hline
\end{tabular}
\caption{\textbf{Comparing multiple-output GP fits for clustered curves.} Sixteen leaves each from two sub-populations (corresponding to the two rows) of the Flavia dataset, modeled using (a) multiple-output GP with no class label kernel, 
and (b) multiple-output GP with coregionalization class label kernel. Each leaf is observed at 10 sample points, selected from the true curve as shown in the top panel. Averaged IMSPE, ESD, and IUEA (Supplementary Material \ref{sec:Metrics}) across curves within each class is reported.
}\label{fig:flavia_clust} 
\end{figure}
\endgroup

Generalization of classical methods to data from different classes is challenging and usually requires specific consideration of sub-population dependence \citep{stocker2021functional,greven2017general}. 
Previous models for functional and curve data, like \citet{gelfand2005bayesian}, propose Gaussian mixtures to account for this dependence.
We can extend our specified multiple-output kernel for GP fitting to encourage similarity of within-class curve fits while allowing substantial shape differences across classes. %

In this setting, observed sample points from underlying curve $\mathbf{f}^{(j)}$ are also supplemeneted with class label $L^{(j)} \in \{ 1,\ldots,r \}$, i.e., curves are drawn from $r$ classes with $r\leq J$, to indicate sub-population membership.
A simple way to specify a kernel which accounts for class labels is by assuming the full kernel has a separable structure, i.e., $K=k \otimes k_D \otimes k_C \otimes k_G$, with elements
$\text{Cov}\left( f^{(j)}_{d}(s_i^{(j)}),f^{(j')}_{d'}(s_p^{(j')}) \right) = k(s_i^{(j)},s_p^{(j')})k_D(d,d')k_C(j,j')k_G(g,g')$,
where $k$, $k_D$, and $k_C$ are specified as in Section \ref{subsec:SingleCurveMOGP}, and $k_G$ is the group-level kernel that compares the class labels $L^{(j)}=g, \ L^{(j')}=g'$ of curves $j, j'$. As with previous dependence sources, a simple choice is to use the intrinsic coregionalization model, defining a $r \times r$-dimensional coregionalization matrix $G$ such that $G_{g,g'}=k_G(g,g')$. %

Figure \ref{fig:flavia_clust} shows two classes of leaf curves from the Flavia dataset, with each class containing 16 leaves, each observed at 10 equally spaced points. %
Column (a) shows model fits when all 32 curves are pooled together, i.e., ignoring class labels. In this case, the simpler-shaped leaves of Class 2 are estimated to have more complex shape, as the joint model borrows information across both classes. Since Class 1 has leaves with more curvature, some of these features are transferred to the fitting of leaves in Class 2, which is also reflected in its class mean IMSPE. However, fitting a MOGP model using the above kernel specification with coregionalization class label kernel %
yields predictive mean curves which more closely resemble corresponding true curves and preserve sub-population structure. %

\section{Conclusion}
Motivated by the presence of noise and sparse data in statistical shape analysis tasks, we formulate a nonparametric multiple-output GP model for fitting closed curves, using a multi-level kernel to incorporate various sources of dependence. The GP model can generatively sample and re-sample closed curves, allowing certain sampling assumptions (e.g., the presence of noiseless dense samples) to be relaxed for applications of curve-based shape analysis methods such as registration, %
curve reconstruction, simultaneous/sequential landmark detection, and sub-population modeling. Theoretical and practical difficulties in GP fitting of closed curves are discussed, motivating the use and constrained estimation of periodic stationary covariance kernels on the input space, combined with coregionalization kernels to account for between-coordinate dependence of a curve as well as between-curve dependence. %
This allows model fitting to exploit the structural similarity of curves in applications. %
We note that our multi-shape GP model extends naturally to closed curves embedded in higher dimensional spaces (rather than $\mathbb{R}^2$), although non-trivial generalizations are needed for modeling closed manifolds with dimension greater than 1.

For future work, we can consider a wider variety of coregionalization kernel construction other than product kernels (e.g., additive, non-stationary and convolution); specifically, we can consider local kernels of specific forms. 
In addition, it would be interesting to explore%
how to incorporate both topological \citep{nigmetov2022topological,luo_generalized_2020,luo2019combining} and geometric information  \citep{solomon2021geometry,chowdhury2021generalized} in real-world applications. Another natural extension is a fully Bayesian version of this GP model, 
which can incorporate curvature and/or shape-informed priors. %

\spacingset{1}
\textbf{Acknowledgments:} HL was supported by the Director, Office of Science, of the U.S. Department of Energy under Contract No. DE-AC02-05CH11231. JS was supported by the Laboratory Directed Research and Development program of Los Alamos National Laboratory under project number 20200065DR. To the best of our knowledge, we claim no conflict of interest with this
manuscript. The code to produce all figures and
experiments will be made available on the authors' website.

\newpage
\renewcommand\thesection{\Alph{section}}
\renewcommand\thesubsection{\thesection.\arabic{subsection}}
\setcounter{section}{0}

\begin{center}
{\Large\textbf{SUPPLEMENTARY MATERIAL}}
\end{center}

\spacingset{1.9}
\section{\label{sec:ill_param}Shape and Parameterizations}

Mathematically, there is a distinction between closed curves and their shapes. Closed curves are 
functions $\mathbf{f}:\mathcal{S}^1 \rightarrow \mathbb{R}^2$.
Shape can be defined in numerous ways, depending on how curves are represented. For this paper, we follow the definition ascribed by elastic shape analysis literature \citep{srivESA}, where shape is a mathematical representation of a curve which is invariant to translation, scale, rotation, and curve parameterization. This is an extension of landmark shape spaces \citep{kendall_shape,dryden2016statistical}, which do not consider parameterization, as shapes are represented by finite point sets. More formally, the shape of curve $\mathbf{f}$ is an equivalence class, containing all curves which only differ by one of the above shape-preserving transformations. The elastic shape space $\mathcal{S}$ is the set of all such equivalence classes. The left panel of Figure \ref{fig:sameshape} shows an example of four closed curves (representing deer from the MPEG-7 dataset) which belong to the same equivalence class (i.e., have identical shape) in $\mathcal{S}$.

Curve parameterizations are important for defining shape, as they drive the matching of features across curves. The middle panel
show two different deer curves parameterized by arc-length. 
Note that the legs are not correctly matched to each other, as these features occur at different arc-length parameter values along the two curves.
The right panel shows correspondences between the same two deer when curve parameterizations are induced by the elastic metric. Note that this appropriately matches the aforementioned legs, and in general results in a more apt comparison between the two deer.

\begin{figure}[h]
\centering \begin{tabular}{c|cc}
\hline
\includegraphics[width=0.3\textwidth]{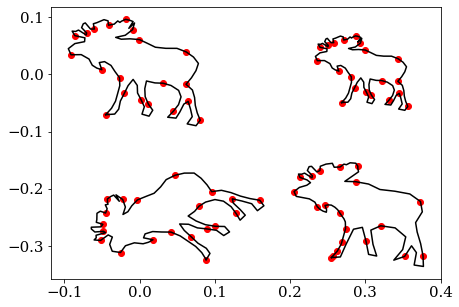} & \includegraphics[width=0.3\textwidth]{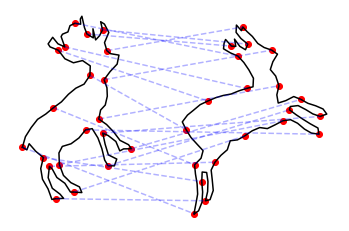} & \includegraphics[width=0.3\textwidth]{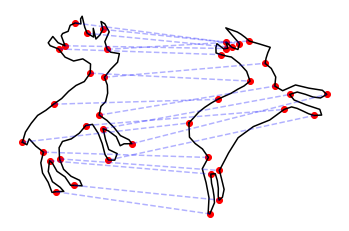}\\
\hline
\end{tabular}\caption{\textbf{Shapes and comparison of parameterizations.} (Left) Four deer curves (in black) from MPEG-7 with equivalent shape, as they only differ by a translation, scaling, rotation, and/or re-parameterization. (Middle / right) Two deer curves, comparing correspondence induced by arc-length parameterizations (middle) to elastic parameterizations (right). Red points are in correspondence across curves (sampled with respect to their parameterizations), with dashed blue lines showing the points in correspondence. }
\label{fig:sameshape}
\end{figure}

\section{Arc-length Parameterization Algorithms}
\label{sec:arc-length Parameterization Algorithms}
In the setting where closed curves are densely sampled, the arc-length parameter is approximated well by equally spaced values on the unit interval assigned to each observed sample point; however, in general, the approximation error of arc-length parameters cannot be ignored. In this section, we present two algorithms to map between observed sample points on closed curves and their corresponding arc-length parameter values. We also provide theoretical error bounds for these algorithms.

Note that these algorithms compute quantities with respect to an underlying curve. For simulations, this can be the (known) true curve, but in practice, this is assumed to be with respect to the piecewise-linear curve constructed from connecting observed sample points. This simple choice is appropriate, as one goal of the model is to estimate the underlying curve, and the ordering of observed sample points reflects the curve's sampling scheme.

\subsection{Enclosing Curves and Parameterization Algorithms}
In Figure \ref{fig:algorithm_xy_arc}, we illustrate Algorithm \ref{alg:xy_to_arc_param}, which converts
the blue point's $x$- and $y$-coordinates (which is not necessarily on the black curve)  into its corresponding arc-length parameter
value represented by yellow shaded
segment. It is numerically computed by identifying a densely ``over-sampled"
black point on the approximating piecewise-linear curve. In this procedure, there are two approximations: the first is in joining observed points by line segments for curve approximation (Step 5 in Algorithm \ref{alg:xy_to_arc_param}), i.e., yellow highlighted and black segments in Figure \ref{fig:algorithm_xy_arc}.
This is intrinsic to the curve.
The second approximation comes from computing ``point-to-segment" distance by dense sampling on the line segment between two consecutive points (Steps 7 to 9 in Algorithm \ref{alg:xy_to_arc_param}), i.e., $N_b$ black dots on the segments between points 3 and 4 in Figure \ref{fig:algorithm_xy_arc}. This depends on the extrinsic metric defined in $\mathbb{R}^2$. 
Algorithm \ref{alg:arc_to_xy_param}
converts the green point with arc-length parameter value represented
by green shaded segment into $x$- and $y$-coordinates (always on
the black curve) with reference to the approximating piecewise-linear curve. In this procedure, only the intrinsic curve approximation is used, and not the extrinsic metric in $\mathbb{R}^2$.

\begin{figure}[t]
\centering 
\includegraphics[width=0.35\textwidth]{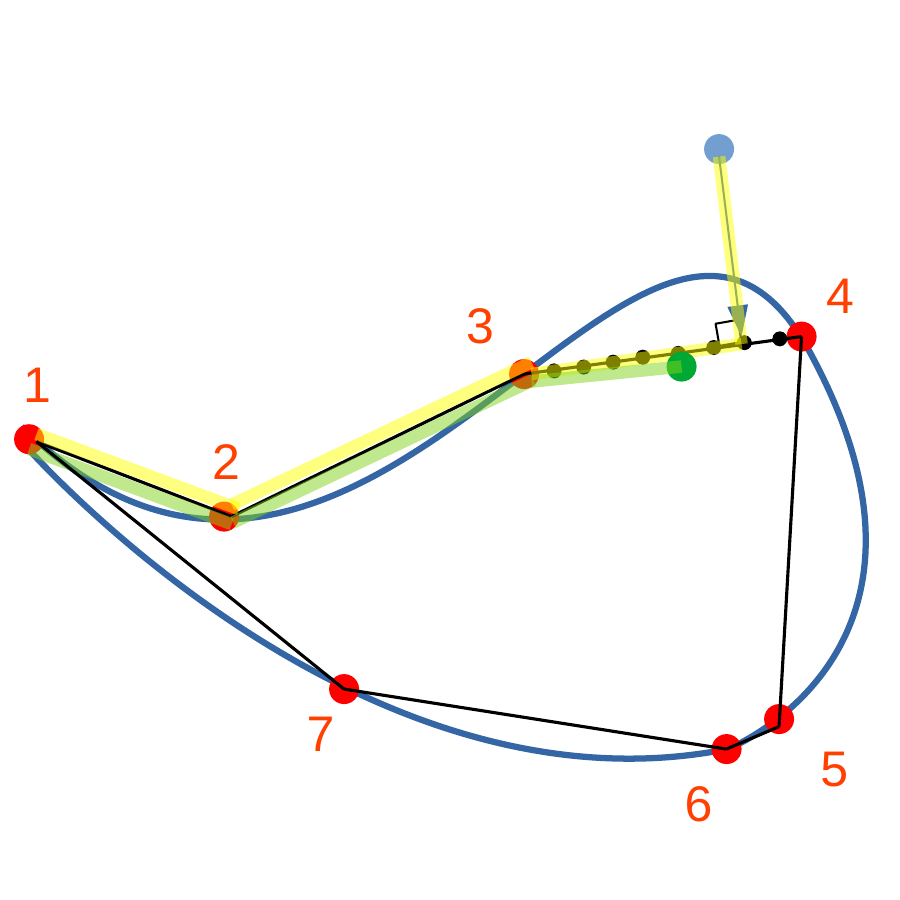}
\caption{\textbf{Arc-length parameterization algorithms.} Seven red points (with ordering displayed) sampled from the underlying blue 
curve; point 1 is the starting point, followed by the points 2, 3, $\cdots$, 7. The approximate  piecewise-linear curve constructed
by connecting these points is in black. 
Black points are sampled from black segments to compute the closest distance from an external point (in green) to the curve. }
\label{fig:algorithm_xy_arc}
\end{figure}

\begin{algorithm}[h!]
\caption{\label{alg:xy_to_arc_param} $\texttt{xy\_to\_arc\_param\_single}$:
converts an observed sample point's $(x,y)$-coordinates (from a closed curve) into its arc-length parameter value}
\KwData{\texttt{osamp} ($2\times (n+1)$ matrix, columns are $(x,y)$-coordinates
for $n$ ordered sample points from a closed curve in $\mathbb{R}^{2}$,
where final column is same as first column)} \KwIn{\texttt{xy\_coord}
($(x,y)$-coordinates for point which arc-length parameter is desired)
} \KwResult{\texttt{arc\_param} (arc-length parameter value for
the target point \texttt{xy\_coord})}

(1) Generate ``over-sampled'' piecewise linear curve \texttt{ospl}
($2\times\left((n+1)(N_{b}+1)-N_{b}\right)$ matrix) by linearly
interpolating $N_{b}$ points between each consecutive pair of points
contained in \texttt{osamp} to find the closest point to the target
point\;

\Indp \Begin{ \textbf{set} $N_{b}$ (e.g., $N_{b}=19$ is sufficient)\;

\tcc{final iteration connects final point to starting point}\
 \For{$i\leftarrow1$ \KwTo $n$}{ $\texttt{ospl}[:,i\cdot(N_{b}+1)-N_{b}]=\texttt{osamp}[:,i]$;\
 \tcc{direction to move to next sample point along segment}\
 $\texttt{vecd}=\texttt{osamp}[:,i+1]-\texttt{osamp}[:,i]$;\\

\For{$j\leftarrow1$ \KwTo $N_{b}$}{ $\texttt{ospl}[:,i\cdot(N_{b}+1)-N_{b}+j]\leftarrow\texttt{ospl}[:,i\cdot(N_{b}+1)-N_{b}+j-1]+\frac{j}{N_{b}+1}\texttt{vecd}$;\
 } }

$\texttt{ospl}[:,(n+1)(N_{b}+1)-N_{b}]=\texttt{osamp}[:,n+1]$;\
 }

\Indm (2) Compute the closest distance between \texttt{xy\_coord}
and each point in \texttt{ospl}\;

(3) Identify the point \texttt{cl\_ospl} along "over-sampled" piecewise
linear curve that is closest to \texttt{xy\_coord}\;

(4) Identify immediately preceding sample point \texttt{prev\_sp}
to \texttt{xy\_coord} along \texttt{ospl} and compute arc-length of
\texttt{ospl} up to \texttt{prev\_sp}\;

\textbf{return} \texttt{arc\_param} $=$ arc-length from starting
point to \texttt{prev\_sp} $+$ arc-length between \texttt{prev\_sp}
and \texttt{cl\_ospl}\\
\end{algorithm}

\begin{algorithm}[ht!]
\caption{\label{alg:arc_to_xy_param} $\texttt{arc\_to\_xy\_param\_single}$:
converts an arc-length parameter along a closed curve
into $(x,y)$-coordinates}
\KwData{\texttt{osamp} ($2\times (n+1)$ matrix, columns are $(x,y)$-coordinates
for $n$ ordered sample points from a closed curve in $\mathbb{R}^{2}$,
where final column is same as first column)} \KwIn{\texttt{arc\_param}
(arc-length parameter value for which $(x,y)$-coordinates are desired)}
\KwResult{\texttt{xy\_coord} ($(x,y)$-coordinates in $\mathbb{R}^{2}$
corresponding to \texttt{arc\_param})}

(1) Compute arc-length parameter values based on \texttt{osamp}\;

\Indp \Begin{ $\texttt{res}[1]=0$ \; \For{$i\leftarrow2$ \KwTo
$n+1$}{ $\texttt{res}[i]\leftarrow\texttt{res}[i-1]+\left|\texttt{osamp}[:,i]-\texttt{osamp}[:,i-1]\right|$\;
} }

\Indm (2) Identify sample point with arc-length parameter just preceding
(or equal to) \texttt{arc\_param}\;

\Indp $\texttt{dst}=\lvert\texttt{arc\_param}-\texttt{res}\rvert$\;

\texttt{previ} $=$ index of sample point $\texttt{prev\_sp}$ preceding
or equal to \texttt{arc\_param} based on \texttt{dst}\;

\Indm (3) Linear interpolation based on arc-length distance to \texttt{prev\_sp}\;

\Indp

\Begin{ $\texttt{prev\_dst}=\left|\texttt{arc\_param}-\texttt{res}[\texttt{previ}]\right|$\;

$\texttt{interval\_dst}=\texttt{res}[\texttt{previ}+1]-\texttt{res}[\texttt{previ}]$\;

$\texttt{rat}=\frac{\texttt{prev\_dst}}{\texttt{interval\_dst}}$;\\

$\texttt{xy\_coord}=\texttt{osamp}[:,\texttt{previ}]+\texttt{rat}\left(\texttt{osamp}[:,\texttt{previ}+1]-\texttt{osamp}[:,\texttt{previ}]\right)$;\\
 }

\Indm \textbf{return} \texttt{xy\_coord}.
\end{algorithm}

The intuition of our algorithm is "to project the point onto the piece-wise linear approximation to the curve". As stated in our algorithm, we enclose the input sample points by replicating the first column to the last column. This is crucial
in estimating the total arc-length.
If we do not enclose the input sample points, then we are treating the sample points as being drawn from an open curve. As a result, any points located between the observed last and first points in the sample will be mapped to either the first or the last point. This results in an error as large as the length of the segment between last and first point. In addition, the presented algorithms behave less consistently, due to dependence on the first and last sampled points. This error is generally small when we have dense sample points (which is the usual setup in elastic shape analysis), yet when there are finitely many sample points, this enclosing step becomes more crucial.

We wrap up this section by presenting the algorithm complexity for conversion algorithms. Although the complexity is straightforward from construction, we point out that these are linear in $n$, the number of sample points.
\begin{prop}
The complexity of Algorithms \ref{alg:xy_to_arc_param} and \ref{alg:arc_to_xy_param} are $O(N_b\cdot n)$ and $O(n)$, respectively.
\end{prop}

\subsection{Curve Length Estimation}
As the number of observed sample points increases, one obtains an increasingly accurate estimate of the underlying curve's arc-length. 
Recall that we use $y_{i1},y_{i2}$ for the $x$ and $y$-coordinates of observed sample point $i$, as in \eqref{eq:coordinate convetion}, and we identify the $(n+1)$-th point with the first point, i.e., $y_{n+1,d} \coloneqq y_{1d}$ for $d=1,2$, based on our closed curve assumption. Then, for underlying curve $\mathbf{f}=(f_1,f_2)$, its total arc-length is given by
$\ell(\mathbf{f})\coloneqq\int_{\mathbf{f}}ds=\int_{(x,y)\in \bm{f}}\sqrt{(\frac{df_1}{ds})^2+(\frac{df_2}{ds})^2} \ ds$. Thus, the absolute error for estimation of its total arc-length is bounded above by: 
\begin{align*}
\textcolor{black}{\Delta} &\textcolor{black}{\coloneqq \sum_{i=1}^{n}\left|\int_{s\in[s_i,s_{i+1}]}\sqrt{\left(\frac{df_1}{ds}\right)^{2}+\left(\frac{df_2}{ds}\right)^{2}}ds- \sqrt{(y_{i+1,1}-y_{i1})^2+(y_{i+1,2}-y_{i2})^2}\right|} \\
& \textcolor{black}{\leq \sum_{i=1}^{n}\left|(s_{i+1}-s_i)\max_{s\in[s_i,s_{i+1}]}\sqrt{\left(\frac{df_1}{ds}\right)^{2}+\left(\frac{df_2}{ds}\right)^{2}}- \sqrt{(y_{i+1,1}-y_{i1})^2+(y_{i+1,2}-y_{i2})^2}\right| \,,}
\end{align*}
where $f_1(s),f_2(s)$ are 
univariate $x,y$-coordinates functions with shared parameter $s$. As in the main text, we assume this parameter is the arc-length parameter, although other choices are allowed. This bound depends on the curvature of $\mathbf{f}(s)$ and the maximal difference between sample point coordinates. 
The total error bound above can be decomposed into error bounds which depend on consecutive ordered sample points \textcolor{black}{$(y_{i1},y_{i2})$} and \textcolor{black}{$(y_{i+1,1},y_{i+1,2})$} with arc-length parameters $s_i,s_{i+1}$: 
\begin{align*}
 \textcolor{black}{\tilde{\Delta}_i} & \textcolor{black}{\coloneqq  \left|(s_{i+1}-s_i)\max_{s\in[s_i,s_{i+1}]}\sqrt{\left(\frac{df_1}{ds}\right)^{2}+\left(\frac{df_2}{ds}\right)^{2}}- \sqrt{(y_{i+1,1}-y_{i1})^2+(y_{i+1,2}-y_{i2})^2}\right|} \\ 
& \textcolor{black}{\leq (s_{i+1}-s_i)\max_{s\in[s_i,s_{i+1}]}\sqrt{\left(\frac{df_1}{ds}\right)^{2}+\left(\frac{df_2}{ds}\right)^{2}} - \sqrt{(y_{i+1,1}-y_{i1})^2+(y_{i+1,2}-y_{i2})^2}} \\
& \textcolor{black}{\leq (s_{i+1}-s_i)\max_{s\in[s_i,s_{i+1}]}\left(\left|\frac{df_1}{ds}\right|+\left|\frac{df_2}{ds}\right|\right) - \sqrt{(y_{i+1,1}-y_{i1})^2+(y_{i+1,2}-y_{i2})^2} \,.}
\end{align*}
The absolute value can be removed because the shortest distance between two points is achieved by a straight line segment connecting them when in $\mathbb{R}^2$. 
The last inequality comes from the fact that $\sqrt{a + b} \leq \sqrt{a} + \sqrt{b}$ for positive $a,b$. The following result ensures that under a reasonable sampling scheme, our algorithm consistently estimates the total curve length.

\begin{theorem}
\label{thm:con_xy_arc}
Suppose that points \textcolor{black}{$(y_{i1},y_{i2})$} are sampled from the underlying curve defined by $(f_1,f_2)$, and both coordinate functions are first-order Lipschitz with constants $C_x,C_y>0$. 
The total estimation error $\Delta=\sum_{i=1}^{n}\tilde{\Delta}_i$ converges to 0 as $\max_{i}\|s_{i+1}-s_i\|\rightarrow 0$.
\end{theorem}
\begin{proof}
We know that both coordinate functions are Lipschitz with constants $C_x,C_y>0$, i.e., $\| \frac{df_1}{ds}(s)-\frac{df_1}{ds}(s') \| \leq C_x\|s-s'\|,\| \frac{df_2}{ds}(s)-\frac{df_2}{ds}(s') \|\leq C_y\|s-s'\|$. 
Then, the pointwise error, as defined above, is bounded by:
\begin{align}
\label{eq:keydelta_i} 
\textcolor{black}{\tilde{\Delta}_i} & \textcolor{black}{\leq 
(C_x+C_y)\cdot (s_{i+1}-s_i) - \sqrt{(y_{i+1,1}-y_{i1})^2+(y_{i+1,2}-y_{i2})^2}} \nonumber\\
& \textcolor{black}{ = (C_x+C_y)\cdot (s_{i+1}-s_i) - \sqrt{(f_1(s_{i+1})-f_1(s_{i}))^2+(f_2(s_{i+1})-f_2(s_{i}))^2}} \,.\end{align}
The equality follows from the additional assumption that points are sampled from the underlying curve, though the algorithm holds without this assumption. From this bound, we immediately see that when $\max_{i}\|s_{i+1}-s_i\|\rightarrow 0$, the first summand in \eqref{eq:keydelta_i} converges to 0 uniformly. The second summand 
also converges to 0 uniformly due to the Lipschitz assumption. Using the dominated convergence theorem, we know the sum $\Delta=\sum_i^n \tilde{\Delta}_i$ also converges to 0.
\end{proof}
This error bound $\tilde{\Delta}_i$ suggests a uniform sampling over
the arc-length space tends to minimize the overall estimation error $\Delta=\sum_{i}\tilde{\Delta}_i$,
since both terms in the definition of $\Delta$ are minimized
by a uniform equally spaced sampling over a curve of fixed length, given a fixed number of sample points. A clustered sampling will maximize both terms for at least
one $i_{0}$ and break the condition that $\max_{i}\|s_{i+1}-s_i\|\rightarrow 0$ in the above consistency result, which supports Algorithm \ref{alg:xy_to_arc_param}. The following corollary follows from Algorithm \ref{alg:arc_to_xy_param}.
\begin{corollary}
\label{cor:arc-param-bound}
Under the same assumption of Theorem \ref{thm:con_xy_arc}, a point \textbf{on} the underlying curve of arc-length $s_0>0$ is mapped to the segment between two points \textcolor{black}{$(y_{n_0+1,1},y_{n_0+1,2})$ and $(y_{n_0,1},y_{n_o,2})$} such that:  \textcolor{black}{$$\sum_{i=0}^{n_0} \sqrt{(y_{i+1,1}-y_{i1})^2+(y_{i+1,2}-y_{i2})^2}\leq s_0 \leq\sum_{i=0}^{n_0+1} \sqrt{(y_{i+1,1}-y_{i1})^2+(y_{i+1,2}-y_{i2})^2} \,,$$} with accumulative error bounded from above by $\sum_{i=0}^{n_0+1}\tilde{\Delta}_i$.
\end{corollary}
Algorithm \ref{alg:arc_to_xy_param} may suffer from accumulative error, but also enjoys consistency, as expressed in terms of error bounds above, when $\max_{i}\|s_{i+1}-s_i\|\rightarrow 0$.

Figure \ref{fig:curve_sample_size_and_scheme} compares a (preferred) equally spaced sampling, defined with respect to the arc-length metric along the underlying curve, to clustered samplings; the latter results in unusual GP fits with substantial uncertainty in regions where the underlying curve is not adequately sampled. 

\begin{figure}[t]
\centering
\begin{tabular}{ccc}
\hline
\includegraphics[width=0.3\textwidth]{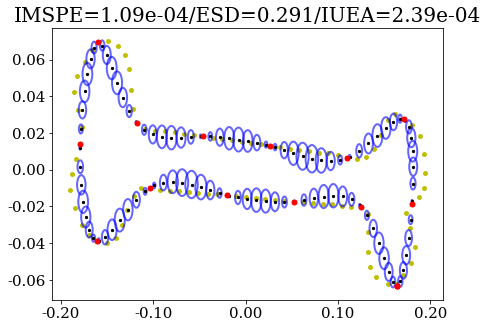}&\includegraphics[width=0.3\textwidth]{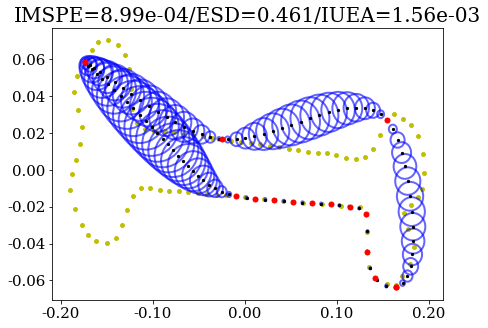}&\includegraphics[width=0.3\textwidth]{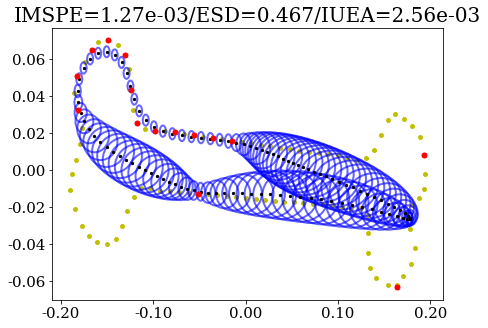}\\
\hline
\end{tabular}
\caption{\textbf{Effect of sampling scheme on GP model fits.} Baseline model fit for a bone curve from MPEG-7, observed at 15 points.
Observed points follow an equally spaced sampling scheme (left), and different cluster sampling schemes (middle and right panels).}
\label{fig:curve_sample_size_and_scheme}
\end{figure}

\section{Gaussian Processes}
\label{sec:SOGP}
Assume we observe $n$ input-output pairs, $(s_{i},y_{i})$, where $s_{i}\in\mathcal{D}$ and $y_{i}\in\mathbb{R}$ (hence "single-output"). 
A Gaussian process (GP) regression model relates input variable $s_{i}$ to corresponding output with noise $y_{i}$ through an unknown target function $f$: 
\begin{align}
y_{i}=f(s_{i})+\epsilon_{i}, \quad \epsilon_{i}\overset{\text{iid}}{\sim}\mathcal{N}_1(0,\sigma_{\epsilon}^{2}), \quad i=1,\ldots,n ,
\end{align}
where $\epsilon_{i}$ represents observation error. In vector notation, we have $\mathbf{y}=f(\mathbf{s})+\boldsymbol{\epsilon}$, where $\mathbf{y}=(y_{1},\ldots,y_{n})^{T}$, $\mathbf{s}=(s_{1},\ldots,s_{n})^{T}$, $f(\mathbf{s})=\left(f(s_{1}),\ldots,f(s_{n})\right)^{T}$, and $\boldsymbol{\epsilon}\sim\mathcal{N}_{n}(\mathbf{0}_{n},\sigma_{\epsilon}^{2}\mathbf{I}_{n})$. The primary goal is to estimate the underlying $f$, which is not assumed to have any particular shape or parametric form.

Following the convention by \cite{Snelson&Ghahramani2006}, we assume a priori that the mean vector $f(\mathbf{s})$ is a realization of a zero mean GP, $f\sim\mathcal{GP}(0,k)$, for covariance kernel $k(\cdot,\cdot):\mathcal{D}\times\mathcal{D}\rightarrow\mathbb{R}$. The kernel $k(s_i,s_{p})=\text{Cov}\left( f(s_i),f(s_{p}) \right)$ models local correlations of the unknown target function $f$, which in turn ultimately controls its smoothness. Samples from this GP follow a joint multivariate normal distribution, $f(\mathbf{s})\sim\mathcal{N}_{n}(\mathbf{0}_{n},\mathbf{K}_{n})$, where $\mathbf{K}_{n}=\left[k(s_{i},s_{p})\right]_{i,j=1}^{n}$ is the $n$-dimensional covariance matrix. GPs are fully specified by the mean and covariance functions, and typically the mean function is set to be identically zero a priori, reflecting lack of knowledge. 

The form of covariance kernel $k$ is often chosen from a specific family (e.g., radial basis function, Matern family) which has its own set of hyperparameters, and the choice of kernel can vary for different applications. For instance, consider the radial basis function (RBF) kernel:
$k_{(\sigma^2,\rho)}(s_i,s_{p}) = \sigma^2 \exp \left\{ -\| s_i-s_{p} \|^2/(2\rho^2)  \right\}$.
This kernel has two hyperparameters: $\sigma^2$ scales the covariance appropriately with output variance, and $\rho$ is the length scale, which controls how quickly the dependence decays. 

We take the frequentist approach in GP model fitting, where kernel hyperparameters are estimated by optimizing the marginal log-likelihood. Bayesian approaches to fitting also exist, but are not investigated in the current work. %
Once these hyperparameters have been estimated, predictions at a new input $s^*$ have mean and variance:\begin{equation}
\begin{gathered} 
\tilde{\mu}(s^*) = \mathbf{K}(s^*,\mathbf{s})\left( \mathbf{K}_n+\sigma^2_{\epsilon}\mathbf{I}_n \right)^{-1}\mathbf{y} \\
\tilde{\sigma}^2(s^*) = k(s^*,s^*)-\mathbf{K}(s^*,\mathbf{s})\left( \mathbf{K}_n+\sigma^2_{\epsilon}\mathbf{I}_n \right)^{-1}\mathbf{K}(s^*,\mathbf{s})^\top,
\end{gathered}
\label{eq:GP_mean_cov}
\end{equation}
where $\mathbf{K}(s^*,\mathbf{s}) = [k(s^*,s_i)]_{i=1}^n$ is a $n$-dimensional vector of kernel evaluations between the new input for prediction and observed inputs. The predictive mean $\tilde{\mu}(s^*)$ is a kernel-weighted combination of observed outputs, and its predictive uncertainty is given by $\tilde{\sigma}^2(s^*)$.

%
\section{Bounds on Periodic Kernels}
\label{sec:proof of main thm}
When defining kernels on arbitrary manifolds, the often-convenient Bochner theorem may not be suitably applied to ensure kernels are positive semi-definite %
\citep{borovitskiy2020matern}; thus, kernel methods, including GP models, cannot be directly applied to closed manifolds. However, with careful constraints, the periodic kernel can be used for closed planar curves with no issue. %

The following result allows us to compare the behavior of a periodic kernel with its non-periodic counterpart. With this result, and our practice that we approximate the period parameter using curve length, we can prove consistency of our model, as shown in Section \ref{sec:proof of consistency thm} below.

\begin{lemma} \label{thm:bounds on periodic kernels} Given true curve length $\ell$, suppose inputs $s_{i},s_{p} \in \mathcal{D}$ are within a period, i.e., $0 \leq r \leq\tau$ for $r=\|s_{i}-s_{p}\|$. Then, 
\[
\sigma^{2}\left( 1-\frac{\pi^2 \ell^2}{4\rho\tau^2} \right)\leq k_{(\sigma^{2},\rho,\tau)}(s_{i},s_{p})\leq\sigma^{2}\left(1+\frac{1}{64}\left(\frac{2\pi^{4}}{\rho^{2}\tau^{4}}+\frac{4\pi^{4}}{3\rho\tau^{4}}\right)\ell^{4}\right) \,.
\]
\end{lemma}

Using this theorem, we can introduce some required and practical periodic constraints between kernel hyperparameters. First, the consideration of consistency requires $\tau \leq\ell$, since we want the dependence to be identical after fully tracing the closed curve. This ensures identifiability of length scale and period hyperparameters. Second, no two points on the curve can be more than arc-length distance $\ell/2$ apart along the closed curve. Thus, heuristically, the correlation should decay to nearly zero before the distance $\ell/2$, to allow two points to be uncorrelated. Finally, from empirical studies, we suggest  $\rho \apprle\frac{\tau}{2}$ to ensure identifiability between these two parameter estimates. For most kernels, we find this constraint to be suitable in practice, as shown by the trade-off between length scale and period in Figure \ref{fig:period_lengthscale comparison}. 

\begin{figure}[t!]
\centering \includegraphics[width=0.75\textwidth]{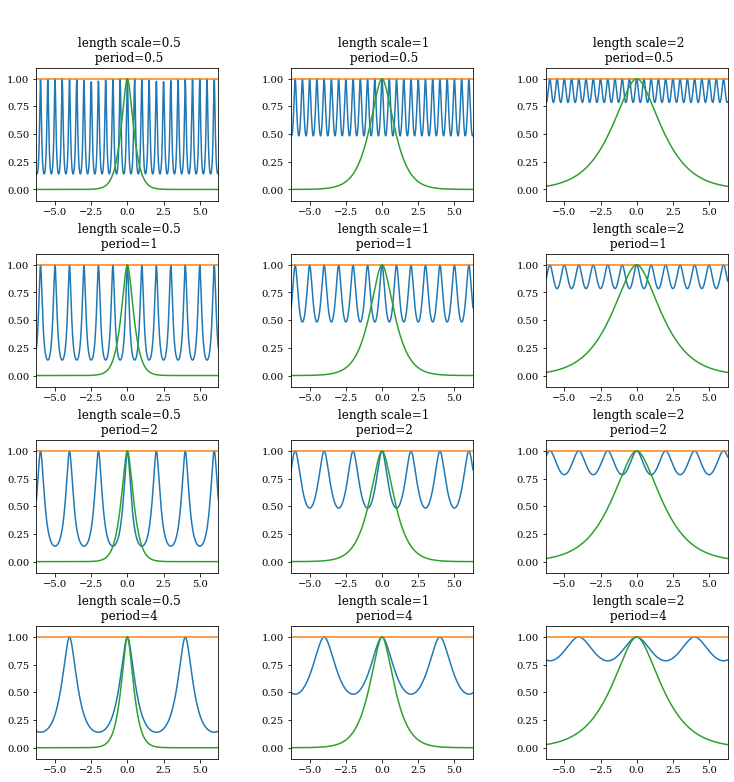}
\caption{\textbf{Trade-off between length scale and period kernel hyperparameters.} Comparison of kernel evaluations for different length scale and period hyperparameters (where applicable). The periodic Matern 3/2 kernel is in blue, the Matern 3/2 kernel in green, and the value 1
in orange (i.e., the value of a degenerate kernel with 0 or infinite
period). The $x$-axis is the arc-length $r$ along the curve between two points $s_i,s_{p} \in \mathcal{D}$ and the $y$-axis is the kernel evaluation $k(0,r)$.) }
\label{fig:period_lengthscale comparison} 
\end{figure}

\begin{proof}
For $r\coloneqq\|s_{i}-s_{p}\|\in(0,\ell/2)$, we first factor out $\sigma^2$ of the stationary periodic covariance kernel $k_{(\sigma^2,\rho,\tau)}(s_i,s_{p})=k_{(\sigma^2,\rho,\tau)}(0,r)$, defined by Equation \eqref{eq:periodic cov ker}. Then, consider the Taylor expansion of this function at $r=0$:
\begin{align*}
\exp\left(-\frac{1}{\rho}\sin^{2}\left(\frac{r}{\tau/\pi}\right)\right) & =1-\frac{\pi^{2}}{\rho\tau^{2}}\cdot r^{2}+\frac{1}{4}\left(\frac{2\pi^{4}}{\rho^{2}\tau^{4}}+\frac{4\pi^{4}}{3\rho\tau^{4}}\right)\cdot r^{4}-O(r^{6})\\
\implies \exp\left(-\frac{1}{\rho}\sin^{2}\left(\frac{r}{\tau/\pi}\right)\right) & \geq1-\frac{\pi^{2}}{\rho\tau^{2}}\cdot r^{2}\\
\implies \exp\left(-\frac{1}{\rho}\sin^{2}\left(\frac{r}{\tau/\pi}\right)\right) & \leq1+\frac{1}{4}\left(\frac{2\pi^{4}}{\rho^{2}\tau^{4}}+\frac{4\pi^{4}}{3\rho\tau^{4}}\right)\cdot r^{4} \,.
\end{align*}
Therefore, 
using the fact that $r =\|s_i-s_p\| < \ell/2$, we obtain the lower bound:
\begin{align}
\begin{split}k_{(\sigma^{2},\rho,\tau)}(s_{i},s_{p}) & \geq\sigma^{2}\left(1-\frac{\pi^{2}}{\rho\tau^{2}}\cdot\|s_{i}-s_{p}\|^{2}\right)\geq\sigma^{2}\left( 1-\frac{\pi^2 \ell^2}{4\rho\tau^2} \right) \,,\end{split}
\label{eq:lowerbound}
\end{align}
Similarly, we get the following upper bound:
\begin{align}
\begin{split}k_{(\sigma^{2},\rho,\tau)}(s_{i},s_{p}) & \leq\sigma^{2}\left(1+\frac{1}{4}\left(\frac{2\pi^{4}}{\rho^{2}\tau^{4}}+\frac{4\pi^{4}}{3\rho\tau^{4}}\right)\cdot\|s_{i}-s_{p}\|^{4}\right)\leq\sigma^{2}\left(1+\frac{1}{64}\left(\frac{2\pi^{4}}{\rho^{2}\tau^{4}}+\frac{4\pi^{4}}{3\rho\tau^{4}}\right)\cdot\ell^{4}\right) \,.\end{split}
\label{eq:upperbound}
\end{align}
\end{proof}

One way to compare the periodic kernel to its non-periodic counterpart is by minimizing the difference between the
upper bound \eqref{eq:upperbound} and the lower bound \eqref{eq:lowerbound}, so that the periodic kernel behaves like its non-periodic counterpart (i.e., squared exponential), such
that a pair of points at maximal distance $\ell/2$ can
attain a small kernel value, and any oscillations of kernel values are dampened within a period. If the difference between upper and lower bounds is large, then there may exist many oscillations within a period, as shown in Figure \ref{fig:period_lengthscale comparison}. On the other hand, we also observe a trade-off between the periodic parameter $\tau$ and the length scale parameter $\rho$ in the upper bound. 
For the upper bound, when we assume that $\tau=\delta\ell$ for some $\delta\approx1$,
the upper bound \eqref{eq:upperbound} reduces to $\sigma^{2}\left(1+\frac{1}{64}\left(\frac{2\pi^{4}}{\rho^{2}\delta^{4}}+\frac{4\pi^{4}}{3\rho\delta^{4}}\right)\right)$.

\section{Consistency Theorem}\label{sec:proof of consistency thm}
We prove the following theorem for a single curve, but expect this consistency result to hold for multiple curves.
\begin{theorem}
Let $f_{1},f_{2}:\mathcal{D}\rightarrow\mathbb{R}$ be the true coordinate functions for a closed curve. Also, let $\hat{f}_{1,n},\hat{f}_{2,n}:\mathcal{D}\rightarrow\mathbb{R}$ be the mean coordinate functions based on a sample of size $n$ on $\mathcal{D}$, as in \eqref{eq:GP_mean_cov}, with additive Gaussian noise (i.e., (4) in \cite{koepernik2021consistency}) parameterized by variance $\sigma_{\epsilon}^{2}$. Assume that the following assumption holds:\newline
(1) $\mathcal{D}$ is homeomorphic to $\mathcal{S}^{1}$, with a metric isometric to the arc-length metric on $\mathcal{S}^1$.\newline
(2) Stationary periodic kernel $k_{(\sigma^{2},\rho,\tau)}$ defined in \eqref{eq:periodic cov ker} has finite, nonzero hyperparameter values $(\sigma^{2},\rho,\tau)$.\newline
(3) The sequence of sample points indexed by sample size $n$ is a dense sequence in $\mathcal{D}$.\newline
Then, the following statements hold:\newline
(A) If $f_{1},f_{2}$ belong to a reproducing kernel Hilbert space \citep{aronszajn1950theory} defined by $k_{(\sigma^{2},\rho,\tau)}$, then $\hat{f}_{1,n}\overset{L_{2}}{\rightarrow}f_{1}$ and $\hat{f}_{2,n}\overset{L_{2}}{\rightarrow}f_{2}$.\newline
(B) If $f_{1},f_{2}$ are continuous with respect to the topology induced by $\|\cdot\|$ in Equation \eqref{eq:periodic cov ker} for $k_{(\sigma^{2},\rho,\tau)}$, then $\sup_{s\in \mathcal{D}}\left|\hat{f}_{1,n}(s)-f_{1}(s)\right|\overset{L_{1}}{\rightarrow}0$ and $\sup_{s\in \mathcal{D}}\left|\hat{f}_{2,n}(s)-f_{2}(s)\right|\overset{L_{1}}{\rightarrow}0$.
\label{thm:consistency thm}
\end{theorem}
\begin{proof}

Recall that the domain we consider for closed curves is
$\mathcal{D}=\mathcal{S}^{1}$. When endowed with an intrinsic
metric, $\mathcal{S}^{1}$ is separable, since we can find points corresponding to rational degrees dense on $\mathcal{S}^{1}$. 
It is also $\sigma$-compact,
since $\mathcal{S}^1$ can be represented as a countable union of closed balls centered
at these points with rational degrees. Any parameterization $s$ that
is isometric to the arc-length metric on $\mathcal{S}^1$ will also induce a separable
and $\sigma$-compact metric space $\mathcal{D}$ . This verifies
the conditions of Theorem 3 in \cite{koepernik2021consistency}
and leads to the $L_2$ convergence results of Statement (A).%

Next, we verify the conditions of Theorem 8 in \cite{koepernik2021consistency}.
In conjunction with Proposition 7 in the same paper, it suffices to verify that the
Minkowski dimension of $\mathcal{D}$ is finite and the kernel $k_{(\sigma^{2},\rho,\tau)}$
is locally Lipschitz. The finiteness of Minkowski dimension follows either by definition (e.g., Definition 3.1 in \cite{falconer2004fractal},
see also \cite{Luo_etal2019} for a discussion for
lower dimensional embedded curves) or by direct computation of the Dudley integral using the natural covering of $\mathcal{S}^1$ consisting of balls centered at rational degrees. Lipschitz continuity follows
from the bounds, specifically \eqref{eq:upperbound}, in the proof of Lemma 
\ref{thm:bounds on periodic kernels}. The periodic
kernel defined in \eqref{eq:periodic cov ker} is
continuous for fixed, estimated finite, non-zero hyperparameters $(\sigma^{2},\rho,\tau)$.
Then, by Proposition 7 in \cite{koepernik2021consistency}, we know that the Dudley integral $J(C,\|\cdot\|)=\int_0^{\infty}\sqrt{\log N(C,\varepsilon,\|\cdot\|)}d\varepsilon<\infty$ for all compact sets in $\mathcal{D}=\mathcal{S}^1$. Inside the integral, the integrand $N(C,\varepsilon,\|\cdot\|)$ is the covering number of $C$ using $\varepsilon$-balls with respect to the metric $\|\cdot\|$ defined by the periodic kernel $k_{(\sigma^{2},\rho,\tau)}$. 

Using Condition (i) of Proposition 4 %
 and Theorem 8 in \cite{koepernik2021consistency},
with all of its assumptions verified above, we can claim that the
posterior GP is continuous for all sample sizes as in Statement (B). %
Furthermore, let $\prod_{n}=\mathcal{GP}(\hat{f}_{d,n},k_{n})$ be the product measure based on mean coordinate function $\hat{f}_{d,n}$ (for $d=1,2$) and covariance function $k_{n}$ (i.e., \eqref{eq:GP_mean_cov}) under a sample of size $n$. 
As an element defined on a functional space, this product measure satisfies $\prod_{n}(U)\overset{L_{1}}{\rightarrow}1$
for every open neighborhood $U$ of the continuous $f_{1},f_{2}$. However, this aspect of consistency is less relevant in our discussion of sampling. 
\end{proof}

\section{Stationary versus Non-Stationary Kernels}
\label{subsec:NonStat}

As mentioned in the main text, the choice of kernel family is important to fitting GP models, as it explicitly controls the smoothness of the underlying function. The left three panels of
Figure \ref{fig:kernel_comp} shows the fits of a butterfly curve under three different periodic stationary kernels: RBF, Matern 3/2, and Matern 1/2, which assume high to low smoothness, respectively, in the underlying curve fit. The RBF kernel fit does not appear to be suitable for modeling this noisily sampled butterfly due to its overly smooth behavior, which results in two self-intersections in the predictive mean curve. The Matern 3/2 kernel yields a fit with the rightmost self-intersection resolved, and increased uncertainty. The Matern 1/2 kernel results in the last smooth predictive mean curve, at the cost of increased uncertainty (as represented by ellipsoids). However, this sacrifice may be worthwhile given the limited number of observed sample points here, as both self-intersections disappear.

An alternative is to use non-stationary kernels \citep{ paciorek2003nonstationary}, which relax the stationarity assumption and increase the flexibility in kernel hyperparameter values, which can vary along the input space. Precisely, a stationary kernel can be expressed as solely a function of the input difference, whereas a non-stationary kernel cannot be. Incorporating periodicity constraints into a non-stationary kernel is not straightforward as it is for stationary kernels; thus, consistency (i.e., Theorem \ref{thm:consistency thm}) cannot be guaranteed, and we cannot expect $\mathbf{f}(0)=\mathbf{f}(\ell)$ in general.
This is evident for the butterfly in the right panel of Figure \ref{fig:kernel_comp}, which uses the non-stationary arc-cosine kernel of order 0, defined in Equation (1) of \cite{Cho2009KernelMF}.  The slight noise perturbation significantly alters estimates of the arc-cosine kernel's hyperparameters, resulting in a curve fit which is not closed. This behavior does not occur as drastically for more densely sampled curves, but we have observed that periodic stationary kernels are more robust in general.

\begingroup
\renewcommand{\arraystretch}{0.5} %
\begin{figure}[t!]
\centering 
\begin{tabular}{cccc}
\hline
periodic RBF & periodic Matern 3/2 & periodic Matern 1/2 & arc-cosine\\
\hline
\includegraphics[width=0.23\textwidth]{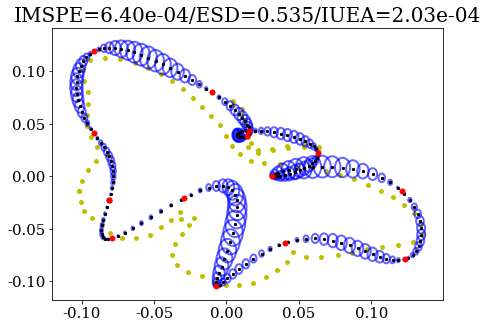}&\includegraphics[width=0.23\textwidth]{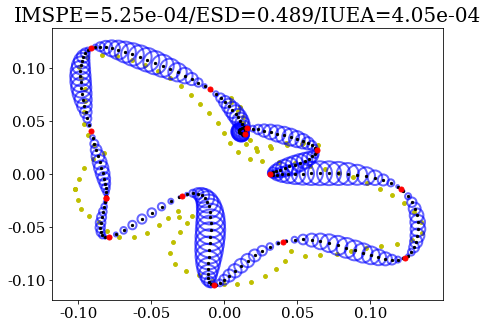}&\includegraphics[width=0.23\textwidth]{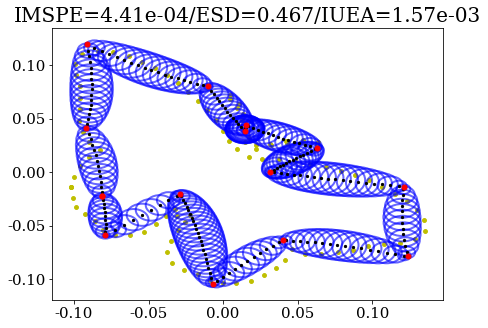}&\includegraphics[width=0.23\textwidth]{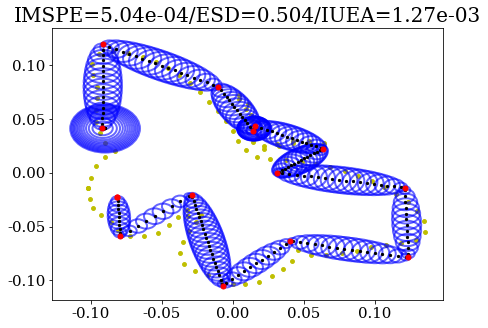}\\
\hline
\end{tabular}
\caption{\textbf{Effect of kernel choices on GP model fits.} Baseline model fit under various periodic stationary (left 3 panels) and non-stationary kernels (right panel) for a butterfly curve from MPEG-7 sampled at 20 points, perturbed pointwise by uncorrelated bivariate normal noise with standard deviation $0.008$.}
\label{fig:kernel_comp}
\end{figure}
\endgroup

The choice of kernel should consider both the sampling scheme of points from the underlying curve, and the smoothness requirement. 
Figure \ref{fig:twotwo_table} shows lizard curve fits for all combinations of either sparse or dense observed sample points under either stationary (periodic Matern 3/2) or non-stationary (arc-cosine of order 0) kernels. If a sparse set of observed points are available, we recommend using a periodic stationary kernel, which guarantees consistency, shows kriging behavior, and can generally result in smoother fits. In comparison, the non-stationary kernel does appear to produce a consistent curve in this case, but the fit is not very smooth given the limited number of observed sample points. On the other hand, if the set of observed points is dense, we recommend either periodic stationary or non-stationary kernels, as both are able to capture finer details of the lizard curve with reduced uncertainty.

\begingroup
\renewcommand{\arraystretch}{0.5} %
\begin{figure}[t!]
\centering 
\begin{tabular}{@{}c@{}|@{}c@{}|@{}c@{}}
\hline
 & \centered{sparse} & \centered{dense}\\
\hline
\centered{stat} & \centered{\includegraphics[width=0.4\textwidth]{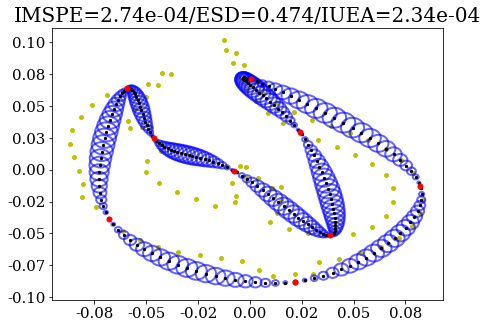}} & \centered{\includegraphics[width=0.4\textwidth]{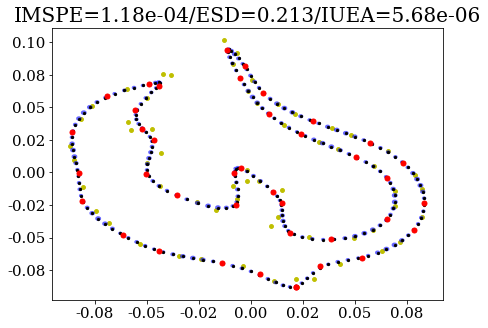}}\\
\hline
\centered{non-stat} & \centered{\includegraphics[width=0.4\textwidth]{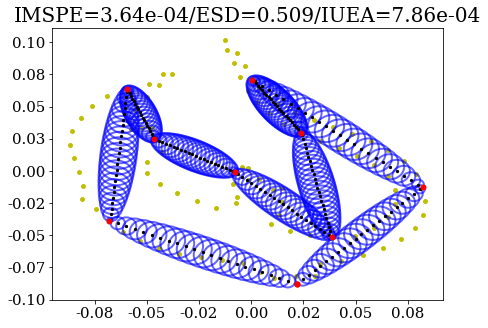}} & \centered{\includegraphics[width=0.4\textwidth]{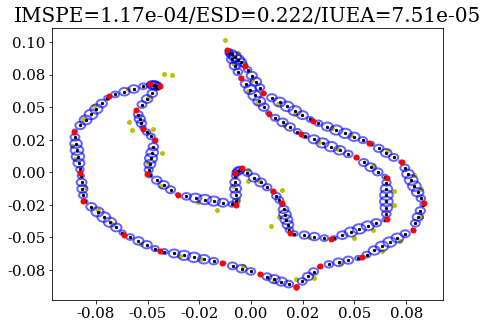}}\\
\hline
\end{tabular}
\caption{\textbf{Comparison of stationary and non-stationary fits by sampling scheme.} Lizard curve from MPEG-7 sampled at either 10 (sparse) or 40 (dense) observed points, and modeled using a multiple-output GP under periodic Matern 3/2 (stationary) or arc-cosine of order 0 (non-stationary) kernels.}
\label{fig:twotwo_table}
\end{figure}
\endgroup

\section{\label{sec:ESA}Elastic Shape Analysis}

Elastic shape analysis \citep{srivESA,kurtek2012statistical} provides a set of tools for modeling the shape, viewed as equivalence classes in a shape space, of curves. An important step within elastic shape analysis is the optimal registration of curves (i.e., by rotation, translation, re-scaling, and re-parameterization) so that points on each curve represented by curve parameter value $s$ are in correspondence. Elastic registration is performed by optimization under what is known as the elastic metric, and typical numerical implementations assume curves are sampled at the same number of points. 
\cite{srivESA} showed that the convenient square-root velocity transformation can be used to simplify calculations of the elastic metric, improving the numerical efficiency of registration.

Formally, given two closed curves $\mathbf{f}^{(j)}: \mathcal{D} \rightarrow \mathbb{R}^2$ for $j=1,2$, the square-root velocity transformation,
\begin{equation}
F\left(\beta^{(j)}(t)\right) \equiv \frac{\dot{\beta}^{(j)}(t)}{\vert \dot{\beta}^{(j)}(t) \vert} \,,
\label{eq:SRVtran}
\end{equation}
is used to define the \emph{square-root velocity function (SRVF)} of the $j^{\text{th}}$ curve $\mathbf{f}^{(j)}(s)$ by $\mathbf{q}^{(j)}(s)=F\left(\mathbf{f}^{(j)}(s)\right)$. Then, pairwise elastic registration of $\mathbf{f}^{(2)}$ to $\mathbf{f}^{(1)}$ means solving the following optimization problem:
\begin{equation}
    (O^*,\gamma^*) = \underset{O \in SO(2),\gamma \in \Gamma}{\text{argmin}}\ \vert\vert \tilde{\mathbf{q}}^{(1)}-O(\tilde{\mathbf{q}}^{(2)} \circ \gamma)\sqrt{\dot{\gamma}} \vert\vert^2 \,,
    \label{eq:ElasticOpt}
\end{equation}
where $\Gamma$ is the space of orientation-preserving re-parameterizations of curve domain $\mathcal{D}$, $\tilde{\mathbf{q}}^{(j)}$ denotes the re-scaling of $\mathbf{q}^{(j)}$ to unit length, $\circ$ denotes function composition, and $\vert\vert \cdot \vert\vert$ denotes the $\mathbb{L}^2$ norm. In practice, this requires curves (and thus corresponding SRVFs) to be sampled by the same number of points, and optimization proceeds by alternating between finding the optimal rotation $\tilde{O} \in SO(2)$ conditional on a fixed re-parameterization $\tilde{\gamma} \in \Gamma$, and then finding the optimal re-parameterization $\tilde{\gamma} \in \Gamma$ conditional on a fixed rotation $\tilde{O}$ until a convergent pair is reached. Further details of this can be found in \cite{srivESA}. For our purposes, it suffices to know that numerical solutions to Equation \eqref{eq:ElasticOpt} rely on curves to be sampled at an equal number of points, and can be poorly conditioned for noisy curves due to the computation of SRVFs, which involves numerical differentiation.

\section{Elastic Shape Registration}
\label{subsec:ElasticShapeReg}
\begin{figure}[t!]
\centering \begin{tabular}{cc}
\hline
\includegraphics[width=0.4\textwidth]{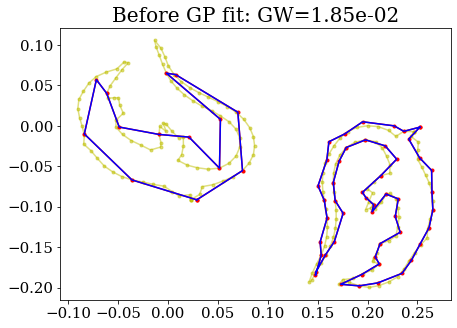}&\includegraphics[width=0.4\textwidth]{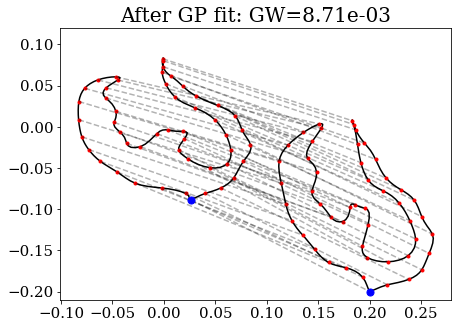}\\
\hline
\end{tabular}
\caption{\textbf{Elastic registration of curves observed at different number of sample points with noise.} (Left) Two lizard curves from MPEG-7, sampled at 15 and 50 points (red), perturbed pointwise by uncorrelated, bivariate normal noise with standard deviation 0.005. (Right) Correspondence (dashed lines) between camels from elastic registration based on predictive mean curves from the jointly fit multiple-output GP, sampled at 100 points each with every 4 points marked red. The matched start point (seed) is in blue.  
}\label{fig:elastic_reg2} 
\end{figure}
Consider the two lizard curves in Figure \ref{fig:elastic_reg2}, sampled at 15 and 50 points, respectively, both equally-spaced and perturbed by bivariate Gaussian noise. 
We fit a multiple-output GP to exploit the structural similarity between the two lizards, obtaining predictive mean curves. These are then densely sampled at the same number of points for registration. Our proposed method is a natural, more robust way %
of drawing an equal number of dense samples along smoother fitted mean curves, as compared to arbitrary re-sampling. The Gromov-Wasserstein (GW) distance (Supplementary Material \ref{sec:Metrics}) between registered curves can be used to confirm this by quantifying ``how closely aligned'' two point clouds are. Unlike elastic shape distance, the GW distance can be computed without relying on curve interpolation. We saw that GP
fitting substantially reduces GW distance for corresponding point clouds and provide more interpretable functional summaries. %
An important task for statistical modeling of curves is shape registration, where the goal is to place curves in an optimal pointwise correspondence with each other by applying various shape-preserving transformations. This usually  %
requires curves to be densely sampled at the same number of points. Fitting the multiple-output GP model, followed by re-sampling the predictive mean, allows for shape registration in a wide variety of more challenging settings, e.g., if the original data is of low quality, sparsely or unevenly sampled as seen in Figure \ref{fig:elastic_reg2} and \ref{fig:curve_rep}. 

\section{Preprocessing for Multiple Curves}
\label{subsec:PreProc}

When considering a joint multiple-output Gaussian process model with between-curve dependence, we recommend the following preprocessing steps prior to fitting, as illustrated by Figure \ref{fig:multishape_comparison_camel}:
\begin{enumerate}
    \item \emph{Centering:} This translates curves to have its center of mass at the origin in Euclidean space. Empirically, this step does not affect the quality of the model, but assists in visualization and exploratory analysis. 
\item \emph{Scaling:} This places curves on the same scale by multiplying by a factor inversely proportional to its estimated total length based on observed sample points. This step affects the absolute values of model parameter estimates.
\item \emph{Rotational and seed alignment:} Since the multiple-output model borrows information across Euclidean coordinates within a fixed curve, it is important to ensure that these dependencies are fairly consistent across curves. Rotating curves to match a pre-specified template as best as possible is one way to achieve this. Within this process, it is also necessary to identify similar ``starting points'' (referred to as ``seeds'') across curves in the vector of observed sample points. This can be performed simultaneously with rotational alignment.
\end{enumerate}

More formally, let $\mathbf{Y}^{(j)}$ be observed sample points from curve $j$. Its centroid can be computed as $\mathbf{\bar{y}}^{(j)} = \left( \bar{y}^{(j)}_{\cdot 1}, \bar{y}^{(j)}_{\cdot 2} \right)$, where $\bar{y}^{(j)}_{\cdot d} = \frac{1}{n_j} \sum_{i=1}^{n_j} y_{id}^{(j)}$ is the average over coordinate $d \in \{ 1,2 \}$. Then, the transformation $\mathbf{Y}^{(j)} \mapsto \mathbf{Y}^{(j)}-\mathbf{\bar{y}}^{(j)}$ shifts the sample points to have centroid at the origin in Euclidean space. 
In landmark shape analysis \citep{kendall_shape,dryden2016statistical}, this step is one way to impose translation invariance, a required first step in mapping point sets to a shape space.

Standardizing the scale of observed sample points is crucial, and is also a necessary step for imposing scale invariance in landmark shape analysis. With respect to GP model fitting, we remarked in Section \ref{subsec:Closed} that for closed curves, a periodic kernel with period close to the true arc-length of the underlying curve yields reasonable closed curve fits. If multiple curves are sharing the same input kernel, inverting the covariance matrix is more numerically stable if their scales are roughly the same.

To standardize the scale 
of observed sample points associated with curve $j$, we divide by the arc-length of the piecewise-linear curve formed by connecting points in $\mathbf{Y}^{(j)}$ by straight line segments (as shown in Figure \ref{fig:algorithm_xy_arc} and Algorithms \ref{alg:xy_to_arc_param}, \ref{alg:arc_to_xy_param}). Let this length be $\ell_j$; then, the mapping $\mathbf{Y}^{(j)} \mapsto \frac{1}{\ell_j}\mathbf{Y}^{(j)}$ re-scales the total length of the segments as the edges of the polygons to be 1. 

Finally, rotational alignment is often desired for both mapping to a landmark shape space as well as visualization purposes. 
Prior to fitting a joint multiple-output GP, we rotate curves $j \geq 2$ to best match a ``canonical'' template curve $j=1$ chosen by a user.  This canonical template curve does not have to be the first curve; for instance, if a well-defined ``average'' curve or some other pre-specified template curve is available, the sample of curves can be rotationally aligned to this template instead.
When performing rotational alignment, there is also a need to identify the ``best'' starting point, i.e., the observed sample point on curves $j \geq 2$ which best matches the first observed sample point on curve $j=1$, which is is known as seed alignment \citep{srivESA}. We choose to perform these steps jointly through an iterative algorithm based on the square-root velocity formulation of \cite{srivESA}. This algorithm is analogous to that of solving \eqref{eq:ElasticOpt}; finding the optimal starting point is part of obtaining $\gamma^*$ in the elastic registration problem, the primary difference is that there is no other re-parameterization performed here.

Let $\mathbf{Y}^{(j)}$ 
be observed sample points associated with curve $j$ (perhaps zero-centered and re-scaled as above). We map $\mathbf{Y}^{(j)} \mapsto \mathbf{Q}^{(j)} = F(\mathbf{Y}^{(j)})$ using the square-root velocity transformation on the piecewise-linear curve $\beta_{\text{pw}}^{(j)}$ through points in $\mathbf{Y}^{(j)}$, i.e., $F\left(\beta_{\text{pw}}^{(j)}(t)\right)$ where $F$ is defined by \eqref{eq:SRVtran}.
For curve $j>2$, we transform $\left(\mathbf{Y}^{(j)}\right)^T \mapsto O_j \pi_j \left( \left(\mathbf{Y}^{(j)}\right) \right)^T$, \begin{equation}
(O_j, \pi_j) = \underset{O \in SO(2), \pi \in \mathcal{P}}{\text{argmin}} \  \vert \mathbf{Q}^{(1)}-O \pi\left( \mathbf{Q}^{(j)} \right) \vert \,,
\label{eq:rotseedopt}
\end{equation}
where $O$ is a $2\times2$ rotation matrix in $SO(2)$, and $\pi$ is a row-shift permutation of $\mathbf{Q}^{(j)}$ (in the set of row-shift permutations $\mathcal{P}$) which simply changes the ``starting point'' for curve $j$. Given a fixed row-shift permutation $\pi$, the optimal rotation is computed explicitly by a singular value decomposition, i.e., $\left(\pi\left(\mathbf{Q}^{(j)}\right)\right)^T \mathbf{Y}^{(1)} = V_j \Sigma_j U_j^T$ yields rotation matrix $O = U_j V_j^T$. We cycle through all possible row-shift permutations in $\mathcal{P}$, compute $O_j$,
and choose the permutation $\pi_j$ with rotation $O_j$ that minimizes the energy in \eqref{eq:rotseedopt}.

\section{Performance Metrics}
\label{sec:Metrics}

To quantify the performance of model fitting and prediction,
we discuss various metrics in this section. The first two compare the model predictions to the ground truth. The third is a useful way to quantify the overall uncertainty and the fourth is useful within the registration task. Their formulation and respective merits are also discussed for reference.

\subsection{IMSPE}
The integrated mean squared prediction error (IMSPE) \citep{plutowski1993cross} characterizes the average discrepancy between a fitted curve and the ground truth curve. It was first developed for density estimation:
\[
\text{IMSPE}(f,\hat{f}) = \mathbb{E} \|\hat{f}-f\|_{2}^{2}=\mathbb{E}  \int_{\mathbb{R}^d} (\hat{f}(x)-f(x))^{2}\ dx \,,
\]
where the $\hat{f}$ is the estimator of true univariate density $f$. When we integrate over the closed curve instead of $\mathbb{R}^d$, this can be taken as the average discrepancy between the GP mean and the corresponding ground truth closed curve. For finite sample of size $n$, we define the overall IMSPE between true closed curve $\mathbf{f}=(f_1,f_2)$ and fitted curve $\mathbf{\hat{f}}=(\hat{f}_1,\hat{f}_2)$ as:
\begin{equation}
\text{IMSPE}(\mathbf{f},\mathbf{\hat{f}}) = \sum_{d=1}^2 \mathbb{E}  \int_{\mathcal{S}^1} (\hat{f}_d(s)-f_d(s))^{2}\ ds \approx \frac{1}{n}\sum_{i=1}^{n} \sum_{d=1}^2 \left(\mathbb{E}\hat{f_d}(s_i)-f_d(s_i)\right)^2 \,.
\label{eq:IMSPE}
\end{equation}
which simply takes the average of the IMSPEs across coordinates. Higher IMSPE values indicate that on average the points on the fitted curve are far from the ground truth curve.

\subsection{ESD}
The elastic shape distance (ESD), rooted in elastic shape analysis, can be used to quantify shape differences between the fitted curve and the ground truth curve. 
Given a true closed curve contour $\mathbf{f}$
and an estimated contour $\mathbf{\hat{f}}$ based on a finite sample of points, with SRVFs $\mathbf{q}$
and $\mathbf{\hat{q}}$ respectively, the elastic shape
distance between the shape classes defined by the respective curves is given by: 
\begin{equation}
\text{ESD}(\mathbf{f},\mathbf{\hat{f}}) =\underset{O\in SO(2),\gamma\in\Gamma}{\text{min}}\cos^{-1}\Big(\langle\langle \mathbf{q}, O(\mathbf{\hat{q}}\circ\gamma)\sqrt{\dot{\gamma}}\rangle\rangle\Big) ,
\label{eq:ESD}
\end{equation}
where $\langle\langle\cdot,\cdot\rangle\rangle$ is the $\mathbb{L}^{2}$
inner product and $SO(2)$ is the special orthogonal group. Note that this is analogous to solving \eqref{eq:ElasticOpt}, and requires numerical optimization. This distance quantifies shape differences which are invariant to translation, rotation, scale, and re-parameterization of the curves being compared. Higher ESD values indicate that the fitted curve has a substantially different shape from the ground truth.

\subsection{IUEA\label{sec:IUEA}}
We define the integrated uncertainty ellipsoid area (IUEA) to quantify overall uncertainty in GP model fitting. This is motivated by uncertainty ellipsoid plots throughout the manuscript, where pointwise uncertainty is depicted by ellipses along the predictive mean curve, with major and minor axes determined by pointwise standard deviations for the two coordinate functions. Similar to IMSPE, we can average over all points on the fitted curve to obtain an overall measure of uncertainty.

Formally, let $\tilde{\sigma}_1$, $\tilde{\sigma}_2$ represent the predictive standard deviations for the fitted closed curve $\mathbf{\hat{f}}=(\hat{f}_1,\hat{f}_2)$. The following derivation is standard, but we cannot find it in the literature; thus, we present it for completeness. Note that at the arc-length parameter value $s$, coordinate pairs follow a bivariate normal distribution:
{\footnotesize
\begin{align*}
\left(\begin{array}{c}
f_{1}(s)\\
f_{2}(s)
\end{array}\right) & \sim \mathcal{N}_{2}\left(\left(\begin{array}{c}
\tilde{f}_{1}(s)\\
\tilde{f}_{2}(s)
\end{array}\right),\left(\begin{array}{cc}
\tilde{\sigma}_{1}^{2}(s) & \tilde{\rho}_{12}(s)\\
\tilde{\rho}_{21}(s) & \tilde{\sigma}_{2}^{2}(s)
\end{array}\right)\right), \quad \tilde{\rho}_{12}(s)=\tilde{\rho}_{21}(s) \,.
\end{align*}}
Therefore, by properties of normal distributions, we can write the event using a chi-squared quantile (with two degrees of freedom):
{\footnotesize
\begin{align*}
\mathbb{P}\left(\left[\left(\begin{array}{c}
f_{1}(s)\\
f_{2}(s)
\end{array}\right)-\left(\begin{array}{c}
\tilde{f}_{1}(s)\\
\tilde{f}_{2}(s)
\end{array}\right)\right]^{T}\left(\begin{array}{cc}
\tilde{\sigma}_{1}^{2}(s) & \tilde{\rho}_{12}(s)\\
\tilde{\rho}_{21}(s) & \tilde{\sigma}_{2}^{2}(s)
\end{array}\right)^{-1}\left[\left(\begin{array}{c}
f_{1}(s)\\
f_{2}(s)
\end{array}\right)-\left(\begin{array}{c}
\tilde{f}_{1}(s)\\
\tilde{f}_{2}(s)
\end{array}\right)\right]\leq\chi_{2,\alpha}^{2}\right) & =1-\alpha \,.
\end{align*}
}
From the Cholesky decomposition, 
\begin{align*}
\left(\begin{array}{cc}
\tilde{\sigma}_{1}^{2}(s) & \tilde{\rho}_{12}(s)\\
\tilde{\rho}_{21}(s) & \tilde{\sigma}_{2}^{2}(s)
\end{array}\right) & =\left(\begin{array}{cc}
\tilde{\sigma}_{1}(s) & 0\\
\frac{\tilde{\rho}_{21}(s)}{\tilde{\sigma}_{1}(s)} & \sqrt{\tilde{\sigma}_{2}^{2}(s)-\frac{\tilde{\rho}_{21}^{2}(s)}{\tilde{\sigma}_{1}^{2}(s)}}
\end{array}\right)\left(\begin{array}{cc}
\tilde{\sigma}_{1}(s) & \frac{\tilde{\rho}_{21}(s)}{\tilde{\sigma}_{1}(s)}\\
0 & \sqrt{\tilde{\sigma}_{2}^{2}(s)-\frac{\tilde{\rho}_{21}^{2}(s)}{\tilde{\sigma}_{1}^{2}(s)}}
\end{array}\right)=\bm{A}\bm{A}^{T} \,,
\end{align*}
we obtain the boundary equation for a $(1-\alpha)\%$ confidence ellipsoid of $(f_{1}(s),f_{2}(s))\in\mathbb{R}^{2}$:
\begin{align*}
\left(\begin{array}{cc}
\tilde{\sigma}_{1}(s) & 0\\
\frac{\tilde{\rho}_{21}(s)}{\tilde{\sigma}_{1}(s)} & \sqrt{\tilde{\sigma}_{2}^{2}(s)-\frac{\tilde{\rho}_{21}^{2}(s)}{\tilde{\sigma}_{1}^{2}(s)}}
\end{array}\right)^{-1}\left[\left(\begin{array}{c}
f_{1}(s)\\
f_{2}(s)
\end{array}\right)-\left(\begin{array}{c}
\tilde{f}_{1}(s)\\
\tilde{f}_{2}(s)
\end{array}\right)\right] & =\sqrt{\chi_{2,\alpha}^{2}} \,.
\end{align*}
The volume of this ellipsoid (i.e., $\left\Vert \frac{1}{\sqrt{\chi_{2,\alpha}^{2}}}\left(\bm{A}^{-1}\mathbf{f}-\bm{A}^{-1}\tilde{\mathbf{f}}\right)\right\Vert _{2}^{2}\leq1$) can
be computed as a scaled unit ball $\mathcal{B}^{2}$ in $\mathbb{R}^{2}$
($\text{vol }\mathcal{B}^{2}=\pi$) \citep{boyd2004convex}: 
\begin{align*}
\sqrt{\chi_{2,\alpha}^{2}}\det\left(\begin{array}{cc}
\tilde{\sigma}_{1}(s) & 0\\
\frac{\tilde{\rho}_{21}(s)}{\tilde{\sigma}_{1}(s)} & \sqrt{\tilde{\sigma}_{2}^{2}(s)-\frac{\tilde{\rho}_{21}^{2}(s)}{\tilde{\sigma}_{1}^{2}(s)}}
\end{array}\right)\cdot\text{vol }\mathcal{B}^{2} & =\pi\sqrt{\chi_{2,\alpha}^{2}}\cdot\left(\tilde{\sigma}_{1}(s)\sqrt{\tilde{\sigma}_{2}^{2}(s)-\frac{\tilde{\rho}_{21}^{2}(s)}{\tilde{\sigma}_{1}^{2}(s)}}\right)\\
 & =\pi\sqrt{\chi_{2,\alpha}^{2}}\cdot\tilde{\sigma}_{1}(s)\tilde{\sigma}_{2}(s)\text{ when }\tilde{\rho}_{21}^{2}(s)=0 \,.
\end{align*}
Therefore, we define the following quantity as our metric:
\begin{equation}
\text{IUEA}(\mathbf{\hat{f}}) = \pi \int_{\mathcal{S}^1} \left(\tilde{\sigma}_{1}(s)\sqrt{\tilde{\sigma}_{2}^{2}(s)-\frac{\tilde{\rho}_{21}^{2}(s)}{\tilde{\sigma}_{1}^{2}(s)}}\right) \ ds \approx \frac{\pi}{n} \sum_{i=1}^n \left(\tilde{\sigma}_{1}(\bar{s})\sqrt{\tilde{\sigma}_{2}^{2}(\textbf{s})-\frac{\tilde{\rho}_{21}^{2}(\bar{s})}{\tilde{\sigma}_{1}^{2}(\bar{s})}}\right) \,, 
\label{eq:IUEA}
\end{equation}
where $\bar{s}\in\mathcal{S}^1$ is determined by the intermediate value theorem on $\mathcal{S}^1$. 

Note that this only accounts for between-coordinate dependence. However, from its derivation, we can see how correlation modeling for between-curve dependence, as described in the main text, helps to reduce uncertainty in model prediction. Also, we note that in the case of separate GPs for each coordinate, the correlation $\tilde{\rho}_{21}(s)=0$ for any $s \in \mathcal{D}$, and the IUEA is proportional to the average product of pointwise standard deviations in the two coordinate directions. In the extreme case where $\tilde{\sigma}_{1}^{2}(s)=\tilde{\sigma}_{2}^{2}(s)=\tilde{\rho}_{12}(s)=\tilde{\rho}_{21}(s)$, the ellipsoid degenerates and its volume is zero, leaving us with the usual notion of the width of a confidence interval. 

\subsection{GW}
As a metric for registration, we use the Gromov-Wasserstein (GW) distance, which considers both overall uncertainty and the  difficulty in alignment of two point sets (with multiplicity). While not directly suitable for comparing two curves, it is widely used for comparing ``how close'' two point clouds (as two degenerate distributions) are. Formally, the $L_2$ Gromov-Wasserstein distance \citep{solomon2021geometry, chowdhury2019gromov} between two probability measures $\mu$  and $\nu$  defined on the same space $M$ is defined as: 
\begin{align}
\text{GW}(\mu ,\nu )\coloneqq\inf _{\gamma \in \Gamma (\mu ,\nu )}\int _{M\times M}\|x-y\|^2_2\, \ \mathrm {d} \gamma (x,y) = \inf _{\gamma \in \Gamma (\mu ,\nu )}\mathbb{E}_{\gamma} \|X-Y\|^2_2 \,,
\end{align}
where $\Gamma (\mu,\nu)$ denotes the collection of all measures on $M\times M$ with marginals $\mu$  and $\nu$  on the first and second component measures. When computing the GW distance between two finite, discrete datasets (as sampled from curves), $\mathbf{f},\hat{\mathbf{f}}$, finite samples are treated as empirical measures consisting of one-point indicators.

We observe that this definition of GW shares similarities to IMSPE, but the distance is optimized over all possible bijective correspondences between point sets (rather than assuming a fixed correspondence). 
GW also shares similarities to ESD, but considers all possible bijections between point sets, as compared to ESD searching over all curve reparameterizations. We mainly use the GW metric for evaluating the registration quality in the main text. Computation of the GW metric is performed in Python using the \texttt{POT: Python Optimal Transport} module \citep{flamary2021pot}, version 0.8.2.

\section{Practical Numerical Issues}
\label{sec:numerical_issues}
\subsection{Ordering and Labeling of Sample Points}
\label{subsec:order}
This section is dedicated to prove that our model is invariant under data formatting, and the sensitivity of our algorithm.
Shape data are assumed to come in the following form:
$\{(y_{11},y_{12}),(y_{21},y_{22}),\cdots,(y_{n1},y_{n2})\}$, where
each sample point is drawn from an underlying curve. The stored $n\times 2$ array  also conveys the order of observed points by the ordering of the $n$ rows in this array.

Since we are considering fitting closed curves, it is important that our proposed model is invariant (or robust) to cyclic permutations (or re-labelings) of the ordering of these observed sample points (i.e., the ordering of the rows in the array), which shifts the point labeled as the ``starting point'' (i.e, the ``seed''). To study this, we propose and perform the following numerical experiment: consider a single camel curve from the MPEG-7 dataset, which we resample via linear interpolation to 1000 points to treat as the ``true'' camel curve. For a fixed number of observed sample points $n$,
we perform the following steps randomly 1000 times:
\begin{enumerate}
    \item Randomly select $n = \{5,20,60\}$ unique points from the camel curve to form the observed sample point matrix $\mathbf{Y}$.
    \item For all $m\leq n!$ cyclic permutations of $\mathbf{Y}$, we obtain a different $n\times 2$ array with a different starting point by applying the permutation. Then, we fit the multiple-output GP model described in Section \ref{subsec:SingleCurveMOGP}, using the periodic kernel with coregionalization. Record kernel hyperparameter estimates for each cyclic permutation: $\sigma^2, \rho, D_{11}, D_{12}, D_{22}$, where $D_{ij}$ denotes the $ij$-entry of the coregionalization matrix $D$. For comparison, we also fit the same model using the equivalent non-periodic kernel with coregionalization, recording these estimated hyperparameters as well.
    \item For each kernel hyperparameter, we compute the standard deviation of the $m$ estimated values associated with the $m$ different cyclic permuted arrays.
\end{enumerate}

\begin{figure}[t]
\centering
\begin{tabular}{ccc}
\hline
\includegraphics[width=0.3\textwidth]{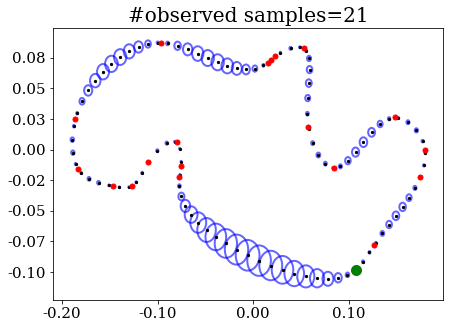} & \includegraphics[width=0.3\textwidth]{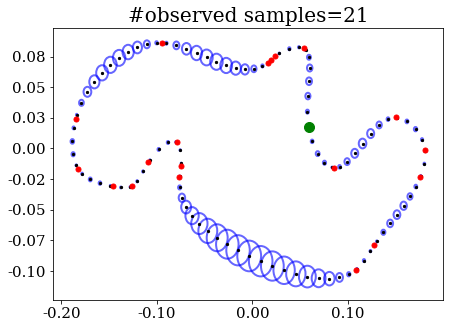} &
\includegraphics[width=0.31\textwidth]{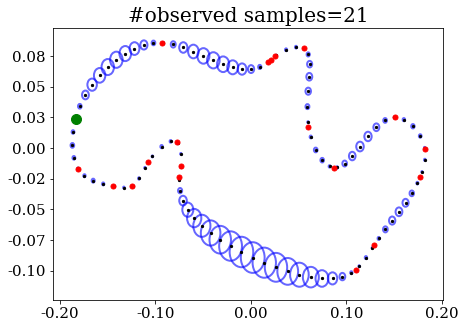}\\
\hline
\end{tabular}
\caption{\textbf{GP model fits under cyclic permutation.} Model fits for a camel curve from MPEG-7, observed at 20 randomly selected points, under a multiple-output GP with a periodic Matern 3/2 kernel. The three plots show fits for the same observed points under three different cyclic permutations of the starting point (in green).}
\label{fig:perm_sp} 
\end{figure}

Figure \ref{fig:perm_sp} shows an example of fits for the same $n=20$ observed camel sample points using $m=3$ different starting points (i.e., three different cyclic permutations). 
Though fits are visually identical (both in predictive mean and uncertainty ellipsoids) across cyclic permutations of the same observed sample point matrix $\mathbf{Y}$, they arise from varying estimates of kernel hyperparameters, as the supplied arc-length parameter values will differ for each permutation.  When there are dense sample points, the total length $\ell$ can be estimated relatively accurately according to Algorithms \ref{alg:xy_to_arc_param} and \ref{alg:arc_to_xy_param}, and our constraint $\tau<\ell$ works better in the sense that the parameter estimate is more concentrated around the true value.

As the number of observed sample points increases, variation in kernel hyperparameter values tend to zero according to Corollary \ref{cor:arc-param-bound} under the periodic kernel. The left plots in Figure \ref{fig:perm_sp2} confirm this, as density estimates of sample standard deviations for all kernel hyperparameters is generally more concentrated around zero for denser samplings (i.e., for larger values of $n$). The only exception is in estimating the length scale hyperparameter, which we suspect is due to its constraint relative to the period (fixed to the estimated total curve length), and the potential shift from interpolating to smoothing regimes of the GP model. The right plots in the figure show the corresponding standard deviations for the equivalent non-periodic kernel, where (i) standard deviations are generally much larger than for periodic counterparts, and (ii) denser samples do not necessarily lead to reduced variation in these hyperparameter estimates.
%
%

\begingroup
\renewcommand{\arraystretch}{0.5} 
\begin{figure}[ht!]
\centering
\begin{tabular}{c|c}
\hline
$\sigma^2$ & $\rho$ \\
\hline
\includegraphics[width=0.45\textwidth]{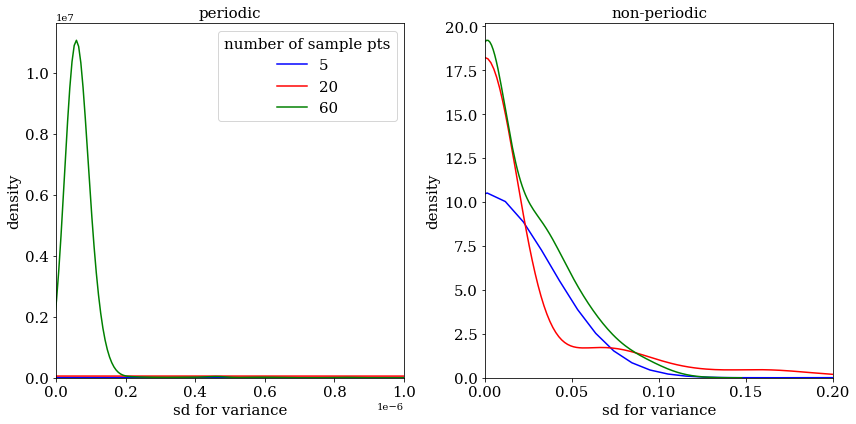} & \includegraphics[width=0.45\textwidth]{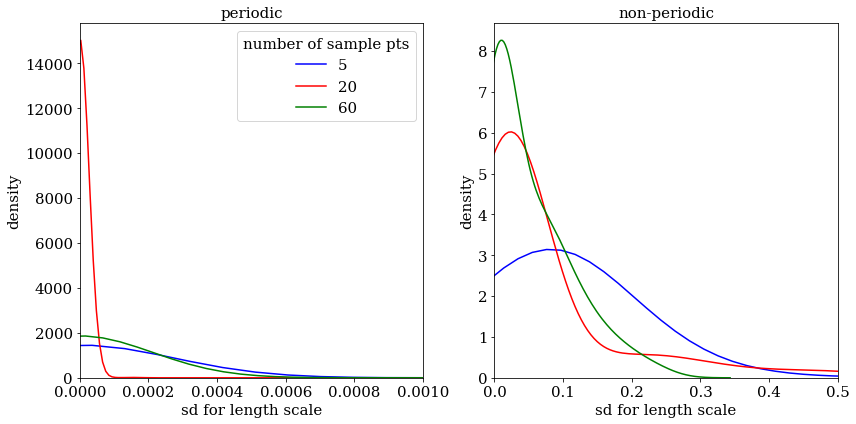}\\
\hline
$D_{11}$ & $D_{12}$\\
\hline
\includegraphics[width=0.45\textwidth]{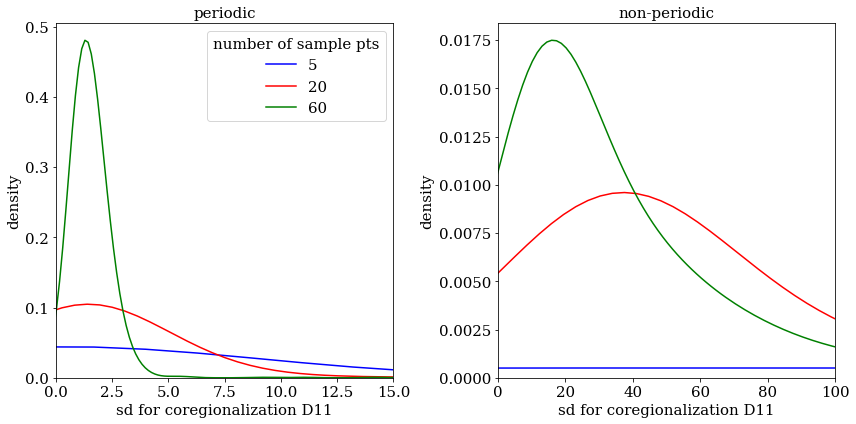} & \includegraphics[width=0.45\textwidth]{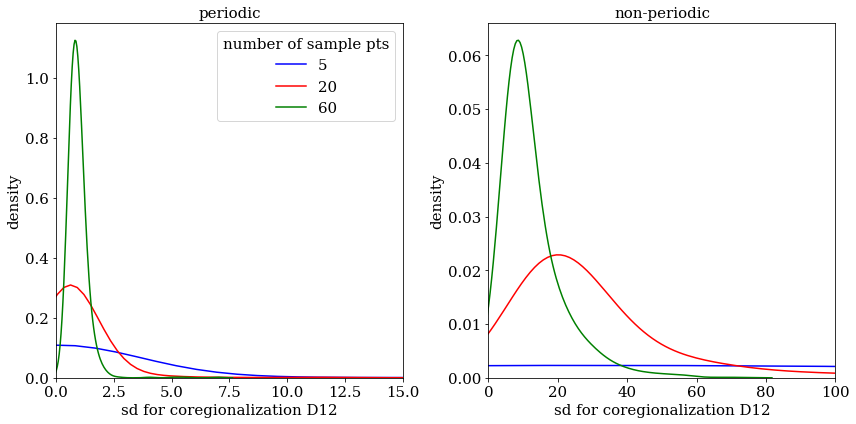} \\
\hline
\end{tabular}
\begin{tabular}{c}
\hline
$D_{22}$\\
\hline
\includegraphics[width=0.45\textwidth]{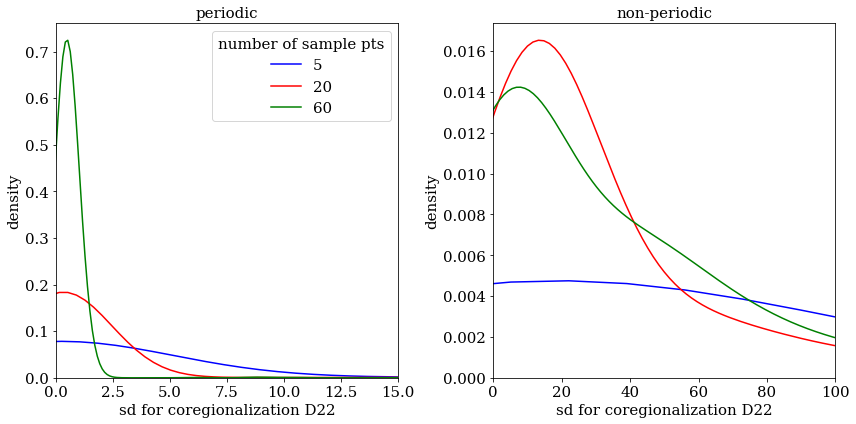}\\
\hline
\end{tabular}
\caption{\textbf{GP model hyperparameter estimates under cyclic permutation.} Kernel density estimates of standard deviations for kernel hyperparameter estimates (top row: variance, length scale; bottom two rows: coregionalization matrix entries $D_{11}, D_{12}, D_{22}$) under all $n$ cyclic permutations of the starting point for 1000 randomly selected sets of $n$ observed sample points on the camel in Figure \ref{fig:perm_sp}. Blue, red, green are for $n=5,20,60$, respectively. Densities are obtained using Gaussian kernels, with common bandwidth across all estimates for a single hyperparameter.
}
\label{fig:perm_sp2} 
\end{figure}
\endgroup

In addition, we also point out that when modeling multiple curves, the encoding method of each curve only affects estimates of kernel hyperparameters. For instance, labeling curves from two different groups as 1 and 2 or as 1 and 3 will lead to the same prediction mean obtained from fitting the GP model. This means that the curve label encoding does not affect the dependence characterized by the between-curve covariance kernel. By default, we use integer labeling (i.e., 0, 1, 2,...) for different outputs, to represent the coordinate function labels and/or sub-population labels. A detailed discussion of how categorical labeling can be modeled in GP regression is in \cite{luo_hybrid_2022}. 

\subsection{Numerical Issues in GP Regression}
\label{sec:general_numericals}

Numerical optimization for maximum likelihood estimation remains partially open at the time this paper is written (especially in high-dimensional domains), including with respect to interpolation error \citep{zaytsev2018interpolation} and simple parameter estimation \citep{basak2021numerical}.  

As pointed out in Section 5.4 of \cite{williams2006gaussian}, when we fit a GP regression model by maximizing the joint likelihood, the multiple local maxima correspond to multiple hyperparameter estimates. Different hyperparameter estimates lead to different interpretations of the same data set. The specific case where the noise variance estimate is small compared to the magnitude of the mean (i.e., large signal-to-noise ratio) is known as kriging \citep{cressie2015statistics}. Kriging shows strong interpolation behavior in that the predictive mean curve from GP model fitting interpolates all data points. On the other hand, when the noise variance is large compared to the magnitude of the mean (i.e., small signal-to-noise ratio), fits exhibit more smoothing behavior. We focus on the kriging aspect of GP regression within this manuscript, and point out that empirically, highly parameterized kernels (e.g., the separable kernel for multiple-output GPs as in Section \ref{subsec:SingleCurveMOGP}) introduce more local maxima in the joint likelihood. 

Gradient-based optimization methods, like L-BFGS-B, for maximizing joint likelihood generally involves specifying a step tolerance, since the GP likelihood can be numerically unstable under a multi-level periodic kernel. Periodicity introduces multiple local optima in the
likelihood function, rendering the optimization problem difficult to
solve. Even if the period is known, we recommend
L-BFGS-B optimization using multiple initial points (also known as ``restarts (of
inner loop)'' in the context of GP optimization).
 Although methods based on auto-differentiation (e.g., tensorflow with Adam)
can improve numerical stability,  
it is more common to execute L-BFGS-B with multiple restart
points.

Simulated annealing, as a global optimizer, seems to be more effective without the need for restarts. However, this approach can be computationally-intensive for GP fitting, especially when there are a quite a few kernel parameters to be estimated.
However, we do note that simulated annealing is important for the landmark estimation task of Section \ref{subsec:lmk}, as the energy function being optimized often has multiple modes. As stated above, landmark estimation using standard gradient-based optimization approaches is not recommended due to its severe dependence on the initial point specified.
The variational GP approach (with natural gradient) seems to work
much less effectively compared to the aforementioned direct solvers based on L-BFGS-B and simulated annealing. 

\subsection{Restricted Kernel Hyperparameter Ranges\label{sec:restricted parameter range}}
As mentioned above, the GP model typically requires some constraints to overcome numerical instabilities during optimization. If one specifies a periodic covariance kernel to fit a closed curve, a practical relationship that bounds the relationship between the length scale and period hyperparameters is discussed 
in Section \ref{subsec:Closed}. In addition, we claimed that the period hyperparameter $\tau$ should be reasonably close to the total curve length $\ell$, based on the proof in Appendix \ref{sec:proof of main thm}. This echoes the observation from our numerical experiments in Figure \ref{fig:single_shape_different_periodicities} that the period $\tau$ should be reasonably close to the total length of the curve $\ell$. Choosing $\tau \ll \ell$ results in poor fits, while $\tau \gg \ell$ typically yields over-smoothed fits.

\begingroup
\renewcommand{\arraystretch}{0.5} %
\begin{figure}[t]
\centering \begin{tabular}{ccc}
\hline
$\text{PM}=0.50$ & $\text{PM}=1.00$ & $\text{PM}=1.50$\\
\hline
\includegraphics[width=0.3\textwidth]{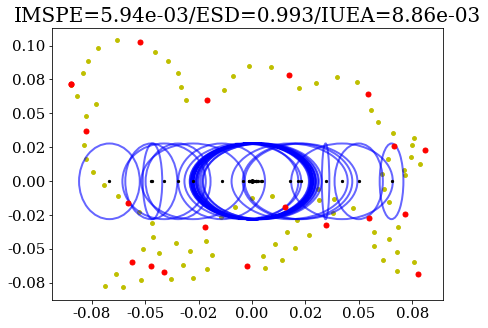} & \includegraphics[width=0.3\textwidth]{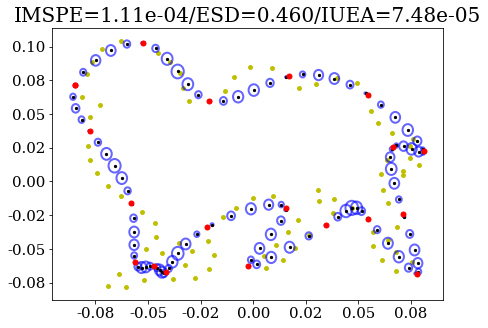} & \includegraphics[width=0.3\textwidth]{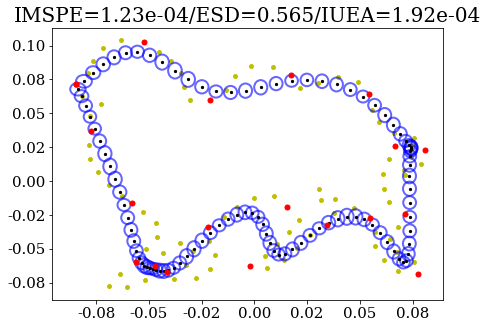}\\
\hline
\end{tabular}\caption{\textbf{Effect of periodicity specification on GP model fits.} Camel curve from MPEG-7, observed at 20 points. The period hyperparameter is fixed to be the estimated arc-length multiplied by a scalar denoted PM, and length scale is constrained accordingly. 
} 
\label{fig:single_shape_different_periodicities}
\end{figure}
\endgroup

For certain smooth kernels, such as the periodic radial basis function kernel, we have observed that it may be necessary to impose an additional constraint that the noise variance is small when fitting the model to obtain a kriging-type fit. as discussed in Supplementary Material \ref{sec:general_numericals}. Fixing this value to be within $(10^{-6},10^{-4})$ tends to resolve the issue in these cases when curves are scaled to unit length. 

If one does not use a periodic covariance kernel to fit a closed curve, we observe inconsistency especially when the arc-length parameter exceeds the estimated total curve length. Even if we want to consider parameters only within a period (arc-length parameter must be less than total curve length, which is unknown), further restrictions of kernel hyperparameters are likely necessary. The range to which hyperparameters are restricted will depend on the scale of the black-box function values. In practice, we assume that both the $x$ and $y$ coordinate function values fall in a finite range, say $\pm10$ after preprocessing. Accordingly, we 
restrict the Gaussian noise variance to $(10^{-6},10^{-4})$ to force the GP regression to interpolate.

\subsection{Adding White Noise or Constant Kernels}
In all examples throughout the paper, the full kernel used is typically the sum of a periodic covariance kernel with a white noise or constant kernel with very small length scale. This is equivalent to adding a diagonal perturbation matrix, and helps resolve numerical issues for inversion of the covariance matrix based on the Cholesky decomposition. In general, adding a constant or white noise kernel with variance $10^{-3}$ is sufficient to avoid these issues. 
Based on experimentation in GP fitting for closed curves, we have found that adding the constant kernel typically resolves these issues more generally than the white noise kernel. 

\subsection{Mappings between Curve and SRVF}
\begin{figure}[t]
\centering
\begin{tabular}{cc}
\hline
\includegraphics[width=0.45\textwidth]{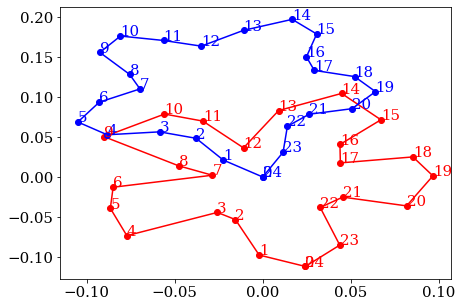}&\includegraphics[width=0.45\textwidth]{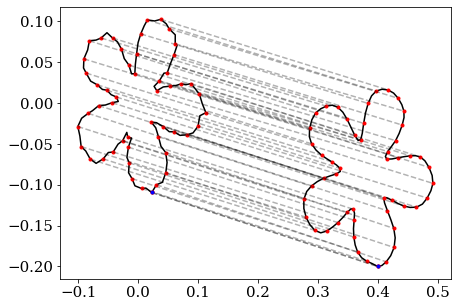}\\
\hline
\end{tabular}
\caption{\label{fig:q2curve_curve2q}\textbf{Numerical issues with mapping between curve and SRVF in \texttt{fdasrsf} module.} (Left) Piecewise-linear curve (red) formed by 25 observed sample points (labeled by their respective ordering) drawn from a flower curve in MPEG-7, and resulting curve (blue) after transforming to its SRVF and then back to a curve. Points are annotated by index. (Right) Correspondence estimated by self-registration of flower curve observed at 101 sample points, perturbed by Gaussian noise, to itself.} 

\end{figure}

We have observed that mapping between curve and SRVF can result in some numerical errors which are exacerbated for curves which either have a relatively small number of observed sample points, or are observed with noise. In particular, the left panel shows a flower with 5 petals from the MPEG-7 dataset, with the red curve showing the original piecewise-linear curve formed by 25 observed points. Using version 2.2.9 of the \texttt{fdasrsf} module in Python for elastic shape analysis \citep{fdasrsf}, we simply apply the function which maps this curve to its SRVF, and then apply the function which maps this SRVF back to its curve representation. The resulting curve, shown in blue, has a drastically different shape than the original curve. In particular, it appears that this mapping tends to inaccurately reconstruct the points of high curvature between petals (e.g., points labeled 2 and 3), essentially smoothing out the resulting curve. 

The right panel shows the correspondence induced by self-registration of the same flower with 5 petals, sampled at 101 points, and perturbed by Gaussian noise with standard deviation 0.008. By self-registration, we mean that we register this noisy curve to itself, with the goal of estimating a re-parameterization function that is the identity (up to small numerical differences) re-parameterization (i.e., a straight line with unit slope). While this re-parameterization function estimated is indeed the identity, the mapping between SRVF and curve representation after this warping is applied has the effect of "smoothing out" the curve. These issues about the intrinsic weakness in the SRVF conversion should be taken into account when considering tasks involving elastic shape analysis using GP regression fits.

\spacingset{1}
\bibliographystyle{chicago}
{
\bibliography{SHAPEFUN}
\spacingset{1}
}
\end{document}